\documentclass[onefignum,onetabnum]{siamart190516}

\usepackage{lipsum}
\usepackage{amsfonts}
\usepackage{graphicx}
\usepackage{epstopdf}
\usepackage{algorithmic}
\usepackage{subcaption}
\usepackage{bm}
\usepackage{bbm}
\ifpdf
  \DeclareGraphicsExtensions{.eps,.pdf,.png,.jpg}
\else
  \DeclareGraphicsExtensions{.eps}
\fi

\newsiamremark{remark}{Remark}
\newsiamremark{hypothesis}{Hypothesis}
\crefname{hypothesis}{Hypothesis}{Hypotheses}
\newsiamthm{claim}{Claim}

\headers{MAP Inference on Random Dot Product Graphs}{D. Wu, D. Palmer, and D. DeFord}

\title{Maximum A Posteriori Inference of Random Dot Product Graphs via Conic Programming\thanks{Submitted to the editors January 6, 2021.
\funding{}}}

\author{David Wu\thanks{Departments of Math and Computer Science, Massachusetts Institute of Technology
  (\email{dxwu@mit.edu}).} \and
  David Palmer \thanks{Geometric Data Processing Group, Massachusetts Institute of Technology
  (\email{drp@mit.edu}).} \and
  Daryl DeFord \thanks{Department of Mathematics and Statistics, Washington State University (\email{daryl.deford@wsu.edu}).}}

\usepackage{amsopn}
\usepackage{amssymb}
\DeclareMathOperator{\diag}{diag}
\newcommand{\1}{\mathbbm{1}}
\newcommand{\RR}{\mathbb{R}}
\newcommand{\PP}{\mathbb{P}}
\newcommand{\EE}{\mathbb{E}}
\DeclareMathOperator{\tr}{tr}
\DeclareMathOperator{\Nul}{Nul}

\newcommand{\norm}[1]{\left\lVert#1\right\rVert}
\DeclareMathOperator{\rank}{rank}

\ifpdf
\hypersetup{
  pdftitle={Maximum A Posteriori Inference of Random Dot Product Graphs via Conic Programming},
  pdfauthor={D. Wu and D. Palmer and D. DeFord}
}
\fi

\usepackage{mathrsfs, mathtools}
\begin{document}

\maketitle

\begin{abstract}
We present a convex cone program to infer the latent probability matrix of a random dot product graph (RDPG). The optimization problem maximizes the Bernoulli maximum likelihood function with an added nuclear norm regularization term. The dual problem has a particularly nice form, related to the well-known semidefinite program relaxation of the \textsc{maxcut} problem. Using the primal-dual optimality conditions, we bound the entries and rank of the primal and dual solutions. Furthermore, we bound the optimal objective value and prove asymptotic consistency of the probability estimates of a slightly modified model under mild technical assumptions. Our experiments on synthetic RDPGs not only recover natural clusters, but also reveal the underlying low-dimensional geometry of the original data. We also demonstrate that the method recovers latent structure in the Karate Club Graph and synthetic U.S. Senate vote graphs and is scalable to graphs with up to a few hundred nodes.
\end{abstract}

\begin{keywords}
  random dot product graph, maximum a posteriori, Bayesian inference, regularization, maximum likelihood estimation, inference, conic programming, convex relaxation, clustering, graph embedding, latent vectors, low rank, consistency
\end{keywords}

\begin{AMS}
  62F15, 90C35, 65F55, 68R10, 15B48
\end{AMS}

\section{Introduction}
Real-world networks appearing in neuroscience, sociology, and machine learning encode an enormous amount of structure in a combinatorial form that is challenging to analyze. A promising analytical approach is to make this structure manifest as geometry. Generative random graph models offer a statistically rigorous way to formulate these problems.

Among the simplest random graph models is the well-known Erd\"{o}s-R\'{e}nyi graph $G(n, p)$, a random graph on $n$ nodes with independent edge probabilities all equal to $p$. Because all vertices look the same in this model, it is too weak to explain common latent structures that appear in real graphs, such as clusters.

It is natural to hypothesize that such structures appear due to underlying affinities between nodes. For example, nodes in a social network represent people who might cluster together based on their shared interests. This hypothesis suggests a natural \emph{geometric} random graph model in which nodes have latent vectors of characteristics, and edge probabilities come from the inner products of these vectors. This model, first studied by Young and Scheinerman \cite{young2007random}, is called the \emph{random dot product graph} (RDPG). In fact, many random graph models, including the Erd\"{o}s-R\'{e}nyi graph and the \emph{stochastic block model} (SBM) \cite{abbe2017community} can be viewed as special cases of the RDPG model.

In this paper, we study the problem of recovering the latent vectors from a graph under the RDPG assumption. While the asymptotic behavior of the RDPG model has been analyzed carefully, we show that with an appropriate prior, the maximum \emph{a posteriori} (MAP) problem can be formulated as a convex cone program, featuring a single semidefinite cone constraint coupled to many exponential cone constraints. This problem is easy to formulate, and it is solvable in polynomial time via commodity conic solvers \cite{mosek,scs}. Moreover, we show that under mild sparsity and nondegeneracy assumptions the solution for a closely related model is asymptotically consistent, and that in practice, our method can reveal a variety of latent geometry in RDPGs. In doing so, we take first steps towards connecting the subfields of latent vector graph models and conic programming.


\subsection{Preliminaries}
Let $G = (V, E)$ be a simple undirected unweighted graph with $n = |V|$ nodes and $m = |E|$ edges. For distinct $i$ and $j$, we write $i \sim j$ if $(i,j) \in E$, and $i \not\sim j$ if $(i,j) \not\in E$. In the RDPG model, every node $i \in V$ in the graph comes with a \emph{latent vector} $X_i \in \RR^d$, and edge probabilities are given by the inner products of these latent vectors, $\PP[i \sim j] = X_i^\top X_j$. We assemble these edge probabilities into the \emph{edge probability matrix} $\bm{P}$. It is important to note that diagonal elements of $\bm{P}$ do not have any probabilistic interpretation since our graphs are assumed to be simple. Conditioned on $\bm{P}$, an RDPG with adjacency matrix $\bm{A}$ has entries $A_{ij}$ that are independent Bernoulli random variables with corresponding parameters $P_{ij}$. That is, we have 
\[
\PP[\bm{A} | \bm{P}] = \prod_{i < j} P_{ij}^{A_{ij}}(1-P_{ij})^{1-A_{ij}}.
\]

An implicit assumption of the RDPG model is that the embedding satisfies the constraints $0 \le P_{ij} \le 1$ for all $i \neq j$. 
One can either view the node embeddings as fixed or themselves drawn from a \emph{latent vector distribution}. Although most prior works use the linear dot product structure, the recent work of O'Connor et al. \cite{o2015maximum} forms the edge probability matrix from the latent vectors using a more general \emph{link function} $\kappa(\RR^d, \RR^d) \to [0,1]$ such as the logistic sigmoid function.

The primary problem of interest is the inference of either the probability matrix $\bm{P}$ or the latent vectors $\bm{X} \coloneqq \begin{pmatrix} X_1^\top & \dots & X_n^\top \end{pmatrix}^\top \in \RR^{n \times d}$ from one sample of $\bm{A}$. Note that for a fixed $\bm{P}$, the latent embedding $\bm{X}$ is only unique up to a $d$-dimensional orthogonal transformation. However, it is still possible to extract an embedding from $\bm{P}$ via a Cholesky decomposition or eigendecomposition.

\subsection{Maximum Likelihood} In this paper we solve an MAP problem with conic programming to recover the probability matrix and latent vectors of an RDPG. In this section, we first formulate the maximum likelihood problem, and then we show that regularization is necessary to recover a nontrivial solution.

Observe that $\bm{P} = \bm{X}\bm{X}^\top$ is positive semidefinite (PSD) by construction. Conversely, if $\bm{P}$ is PSD then it can be decomposed as $\bm{X}\bm{X}^\top$ for some $\bm{X} \in \RR^{n\times k}$ where $k \leq n$. Given an adjacency matrix $\bm{A}$, this leads to the following maximum likelihood problem \cref{prob:mle}:
\begin{equation}\tag{MLE}\label{prob:mle}
\begin{alignedat}{3}
\max_{\bm{P}}              &\quad&  \sum_{i \sim j} \log P_{ij} + \sum_{i \not\sim j} \log(1 - P_{ij})  &&          & \\
\text{subject to: } &\quad&  0 \le P_{ij} \le 1    && &\quad \text{ for all } i \neq j\\
                    &\quad&  \bm{P} \succeq 0,  &&   &\quad 
\end{alignedat}
\end{equation}
where the notation $\bm{P} \succeq 0$ means that $\bm{P}$ is PSD. It is important to note that the diagonal entries of $\bm{P}$ are unconstrained. This makes \cref{prob:mle} highly underconstrained---as we show below, the solution reproduces $\bm{A}$ with nonzero diagonal entries to ensure semidefiniteness. 

\begin{proposition}\label{prop:exact}
There exists an optimal solution $\bm{P}^*$ of \cref{prob:mle} whose off-diagonal entries coincide with those of $\bm{A}$.
\end{proposition}
\begin{proof}
  For feasible $P_{ij}$, we have $\log P_{ij} \le 0$, with equality achieved at $P_{ij} = 1$. Similarly, $\log(1-P_{ij}) \le 0$, with equality achieved at $P_{ij} = 0$. Thus, the objective function can be maximized by setting $P_{ij} = 1$ if $i \sim j$ and $P_{ij} = 0$ if $i \not\sim j$, so long as the semidefiniteness constraint is satisfied. As the diagonal of $\bm{P}$ is unconstrained, we can ensure $\bm{P}^* \succeq 0$ by shifting the diagonal. In particular, $\bm{P}^* = \bm{A} + \lambda \bm{I} \succeq 0$ for sufficiently large $\lambda$. 
\end{proof}

\subsection{Regularization and MAP estimate}
In view of \cref{prop:exact}, it is prudent to regularize \cref{prob:mle} to obtain a nontrivial solution. We do so with the hope of encouraging low rank solutions, which are of particular theoretical and practical interest. For example, low rank solutions are easier to visualize. In addition, low-rank solutions to semidefinite optimization problems are often more efficient to compute via specialized algorithms that take advantage of the low rank, such as the B\"urer-Monteiro method and extensions \cite{burer2003nonlinear, boumal2016non, cifuentes2019burer}. One might hope that such low-rank methods would extend to the optimization problems featured in this paper.

With the goal of a low-rank estimate in mind, we choose a prior that encourages the vector of eigenvalues of the probability matrix to be sparse. One such prior is the Laplace distribution on $\tr \bm{P}$. This aligns with a common heuristic for encouraging low rank solutions, which penalizes the \emph{nuclear norm} $\norm{\bm{P}}_* \coloneqq \sum_i \sigma_i$, where $\sigma_i$ are the singular values of $\bm{P}$ \cite{recht2010guaranteed, davenport2016overview}. Some justification for the heuristic can be provided by noting that the nuclear norm is the convex envelope of the rank function, i.e. $\rank \bm{P} \ge \frac{\norm{\bm{P}}_*}{\norm{\bm{P}}}$, where $\norm{\bm{P}}$ denotes the operator norm of $\bm{P}$.
The nuclear norm of a PSD matrix is just its trace. This leads to the following refinement \cref{prob:reg} of \cref{prob:mle}, which can be interpreted as an MAP estimation problem under the aforementioned Laplace prior, parametrized by a real scalar $C$: 
\begin{equation}\tag{REG}\label{prob:reg}
\begin{alignedat}{3}
\max_{\bm{P}}              &\quad&  \sum_{i \sim j} \log P_{ij} + \sum_{i \not\sim j} \log(1-P_{ij}) - C\tr \bm{P} &&          & \\
\text{subject to: } &\quad&  0 \le P_{ij} \le 1    && &\quad \text{ for all } i \neq j\\
                    &\quad&  \bm{P} \succeq 0  &&   &\quad 
\end{alignedat}
\end{equation}
As with any inference technique, the MAP estimate $\bm{P}^*$ varies with the observed realization $\bm{A}$; given a nondegenerate true edge probability matrix $\bm{P}$, \emph{any} simple undirected graph is realizable, including the empty graph. As such, we cannot expect to prove deterministic bounds on inference performance; probabilistic results are more reasonable.

The constraints of \cref{prob:reg} are of semidefinite program (SDP) type, but the objective function is nonlinear. The problem can be rewritten as a convex cone program. Introducing new variables $\alpha_{ij} \le \log P_{ij}$ for $i \neq j$ and $\beta_{ij} \le \log(1-P_{ij})$ for $i \neq j$, and noting that at optimality these inequalities turn into equalities, we can express the original objective function in a manifestly convex form. We can encode these inequalities with exponential cones. Conveniently, including these inequalities for all $i \neq j$ also automatically enforces the inequality constraints $0 \le P_{ij} \le 1$. We state the definitions of the primal and dual exponential cones below for reference \cite{mosek}.
\begin{definition}[exponential cones]\label{def:exp}
The \emph{primal exponential cone} is the set 
\[ K_{\mathrm{exp}} \coloneqq \{(x, y, z) \in \RR^3 : x \ge y\exp(z/y), y > 0\} \cup \{(x, 0, z): x \ge 0, z \le 0\}. \]
The \emph{dual exponential cone} is its conic dual
\[ K_{\mathrm{exp}}^* \coloneqq \{(u, v, w) \in \RR^3 : u \ge -w\exp(v/w - 1), u > 0, w < 0\} \cup \{(u, v, 0) : u, v \ge 0\}. \]
\end{definition}

Applying these transformations, we obtain the following convex cone program, equivalent to \cref{prob:reg}:
\begin{equation}\tag{CREG}\label{prob:creg}
\begin{alignedat}{3}
\max             &\quad&  \sum_{i \sim j} \alpha_{ij} + \sum_{i \not\sim j} \beta_{ij} - C\tr\bm{P} &&          \\
\text{subject to: } 
                    &\quad&  \bm{P} \succeq 0  &&   \\
                    &\quad& (P_{ij}, 1, \alpha_{ij}) \in K_{\mathrm{exp}} && \quad \text{for all } i \neq j\\
                    &\quad& (1-P_{ij}, 1, \beta_{ij}) \in K_{\mathrm{exp}} && \quad \text{for all } i \neq j 
\end{alignedat}
\end{equation}

The cone program \cref{prob:creg} is the main subject of study in this paper. Moreover, it can be easily implemented with convex optimization interfaces such as CVXPY \cite{diamond2016cvxpy,agrawal2018rewriting} that rely on solvers that support both semidefinite and exponential cone programming such as SCS \cite{scs} and MOSEK \cite{mosek}. As we will see later, the dual conic program turns out to have an appealing interpretation. Finally, we remark that given the binary matrix $\bm{P}^*$ of \cref{prop:exact}, one approach (quite different from ours) that is particularly suited for clustering tasks is to attempt to factor $\bm{P}^*$ into the product of two low-rank binary matrices \cite{zhang2007binary}. However, since binary edge probability matrices are degenerate from the generative perspective, our cone program is more appropriate for the RDPG model. Furthermore, our method recovers richer geometrical structure than simple clusters.

\subsection{Related work}
Prior work on RDPG latent vector recovery has focused on spectral methods such as the \emph{adjacency spectral embedding} (ASE) and \emph{Laplacian spectral embedding} (LSE); see \cite{athreya2017statistical} for a comprehensive review. Assuming the latent vectors reside in $\RR^d$, the ASE entails taking the top-$d$ eigenvector approximation of the adjacency matrix $\bm{A}$. It is known that this provides a consistent estimate of the latent vectors, up to an orthogonal transformation \cite{tang2013universally}. Furthermore, asymptotic normality results exist for the ASE \cite{athreya2016limit} and LSE \cite{tang2018limit}. In addition to these theoretical results, these spectral techniques have been applied to graph machine learning, 
including vertex nomination \cite{yoder2020vertex,agterberg2020vertex}, vertex embedding and classification \cite{sussman2014consistent,chen2016robust}, and joint embedding \cite{wang2019joint} problems. 

One limitation of the ASE is the requirement that the embedding dimension $d$ is known in advance. In order to address these limitations, Yang et al. \cite{yang2020simultaneous} have recently proposed a framework for simultaneously discovering the rank and number of clusters for SBMs using an extended ASE and the Bayesian information criterion. We instead propose a cone program with an adjustable regularization hyperparameter, so that our method does not require \emph{a priori} knowledge of $d$. Empirically, we show that our method (which applies for general RDPGs) leads to a cross validation procedure for rank discovery. Furthermore, our method outperforms the vanilla ASE on clustering tasks and spectral norm distance, even when the true $d$ is known.

Maximum likelihood (ML) approaches to recovering matrix structure or solving graph problems are common in the literature. Arroyo et al. \cite{arroyo2018maximum} studied the graph matching problem in corrupted graphs via ML estimators. For the specific problem of latent vector recovery, Choi et al. \cite{choi2012stochastic} studied ML estimation for the SBM and showed that the when the number of blocks is $O(n^{1/2})$, the fraction of misclassified nodes converges in probability to 0. For latent vector graphs with a logistic sigmoid link function, O'Connor et al. \cite{o2015maximum} give a spectral approximation algorithm that is asymptotically equivalent to ML latent vector recovery, but do not directly study consistency results. Our cone program does not directly depend on any spectral algorithms, but we are still able to study consistency and likelihood properties. 

This work joins a body of literature that applies convex relaxation methods to broader statistical inference problems, such as Logistic PCA \cite{collins2002generalization,landgraf2015dimensionality}, angular synchronization \cite{singer2011angular}, and compressed sensing \cite{davenport20141, cai2013max}. In fact, semidefinite relaxations, which are special cases of cone programs, have been used to study SBM inference \cite{abbe2017community}. The convex relaxation method has also been fruitfully applied to low rank matrix recovery, by relaxing the nonconvex rank function to a nuclear norm or max norm \cite{davenport20141, cai2013max}. Standard statistical learning theory arguments are used to bound the generalization error. Equipped with our cone program relaxation, we demonstrate a novel application of tools from conic programming to the problem of RDPG inference.

Our proof technique for consistency bounds takes inspiration from the matrix completion literature, but some care must be taken because of the failure of a Lipschitz-like criterion.  Our approach also adds a trace penalty term rather than restricting the feasible set, which leads to looser requirements for the true $\bm{P}$ when formulating consistency results. 

\subsection{Main contributions}
We summarize our main contributions below. We lay down theoretical foundations for the RDPG inference problem using a convex cone programming approach. 
We believe that the inference and conic optimization perspective offers insights that have not been leveraged in the study of these latent vector models. Our analysis, which centers around the study of our cone program's optimality conditions, yields interesting theoretical properties of the inference problem. 

After taking the dual in \cref{sec:dual}, we use the primal-dual optimality conditions to derive an explicit relationship between the primal and dual optimal solutions $\bm{P}^*$ and $\bm{Q}^*$ (see \cref{prop:algebraic-rel,prop:main-bound}). Based on our bounds on the entries of $\bm{P}^*$ and $\bm{Q}^*$, we attempt to control their ranks. Unfortunately, we demonstrate a counterexample and a fundamental barrier to using our proof technique to prove that $\rank \bm{P}^*$ decreases arbitrarily as $C \to \infty$ (see \cref{prop:rank-barrier}). On the other hand, we see in \cref{subsec:likelihood,subsec:generalization} that our optimality bounds lead to asymptotic likelihood bounds and consistency bounds for a slightly modified model (see \cref{thm:likelihood-bound,thm:generalization}, respectively). 

In \cref{sec:experiments}, we show empirical evidence that our inference approach indeed recovers fine-grained latent vector geometry. Moreover, our algorithm's single regularization hyperparameter enables the use of simple cross-validation and a parameter sweep to \emph{discover} the rank of $\bm{P}^*$. Our cross-validation experiments provide strong evidence for uniqueness of the optimal regularization parameter and rank. Discussion and conclusions follow in \cref{sec:conclusions}.

\section{Dual cone program}\label{sec:dual}
In order to analyze the optimal solutions of \cref{prob:creg}, it is fruitful to analyze its dual program. Proofs of the following propositions are provided in \cref{app:a}. As a first step, we verify Slater's condition, which is sufficient for strong duality \cite{boyd_vandenberghe_2004}:
\begin{lemma}\label{lemma:slater}
The feasible region of \cref{prob:creg} has an interior point.
\end{lemma}
\begin{proof}
Set $\alpha_{ij} = \beta_{ij} = -1$ for all $i \neq j$, $P_{ij} = \frac{1}{2}$ for all $i \neq j$ and $\diag \bm{P} = \lambda$. For sufficiently large $\lambda$, $\bm{P}$ is positive definite and furthermore $-1 < \log \frac{1}{2}$, so Slater's condition holds.
\end{proof}
\begin{proposition}\label{prop:dual}
The following problem is the strong dual to \cref{prob:creg}:
\begin{equation}\tag{DREG}\label{prob:dual}
\begin{alignedat}{3}
\min              &\quad&  \sum_{i \sim j} s_{ij} + \sum_{i \not\sim j} v_{ij} + \sum_{i \neq j} u_{ij} &&          & \\
\text{subject to: } &\quad& s_{ij} \ge -1 - \log r_{ij}    && &\quad \text{ for all } i \sim j\\
                    &\quad&  v_{ij} \ge -1-\log u_{ij}     && &\quad \text{ for all } i \not\sim j\\
                    &\quad&  u_{ij} - r_{ij} = Q_{ij}     && &\quad \text{ for all } i \neq j \\
                    &\quad&  u_{ij}, r_{ij} \ge 0    && &\quad \text{ for all } i \neq j  \\
                   &\quad&  \diag \bm{Q} = C    && &  \\
                   &\quad&  \bm{Q} \succeq 0   && &
\end{alignedat}
\end{equation}
At optimality, \cref{prob:creg} and \cref{prob:dual} satisfy the following KKT conditions \cite{boyd_vandenberghe_2004}. Let $(\bm{P}^*, \alpha_{ij}^*, \beta_{ij}^*)$ and $(\bm{Q}^*, \bm{\lambda}_{ij}^*, \bm{\nu}_{ij}^*)$ be a pair of optimal primal and dual solutions to \cref{prob:creg} and \cref{prob:dual}, where $\bm{\lambda}_{ij}^* = (r_{ij}^*, s_{ij}^*, t_{ij}^*)$ and $\bm{\nu}_{ij}^* = (u_{ij}^*, v_{ij}^*, w_{ij}^*)$. Then for all $i$ and $j$ we have
\begin{alignat}{3}
\bm{0} &= \bm{P}^*\bm{Q}^* && \label{eq:kkt1}\\
0 &= \bm{\lambda}_{ij}^* \cdot (P_{ij}^*, 1, \alpha_{ij}^*) &&= r_{ij}^* P_{ij}^* + s_{ij}^* + t_{ij}^* \alpha_{ij}^* \label{eq:kkt2}\\
0 &= \bm{\nu}_{ij}^* \cdot (1-P_{ij}^*, 1, \beta_{ij}^*) &&= u_{ij}^* (1 - P_{ij}^*) + v_{ij}^* + w_{ij}^* \beta_{ij}^*.\label{eq:kkt3}
\end{alignat}
\end{proposition}
We can further reduce the dual as follows:
\begin{lemma}\label{lemma:equiv}
The following relations hold at optimality:
\begin{enumerate}
    \item If $i \sim j$ and $Q_{ij}^* \le -1$, then $u_{ij}^* = 0$ and $r_{ij}^* = -Q_{ij}$.
    \item If $i \sim j$ and $Q_{ij}^* > -1$, then $u_{ij}^* = Q_{ij}^* + 1$ and $r_{ij}^* = 1$.
    \item If $i \not\sim j$ and $Q_{ij}^* < 1$, then $u_{ij}^* = 1$ and $r_{ij}^* = 1 - Q_{ij}^*$.
    \item If $i \not\sim j$ and $Q_{ij}^* \ge 1$, then $u_{ij}^* = Q_{ij}^*$ and $r_{ij}^* = 0$. 
\end{enumerate}
Imposing these relations explicitly, we arrive at the following reduced formulation, which agrees with \cref{prob:dual} at optimality:
\begin{equation}\tag{DCREG}\label{prob:dualopt}
\begin{alignedat}{3}
\min \quad             & \sum_{\substack{i \sim j \\ Q_{ij} \ge -1}} Q_{ij} +\sum_{\substack{i \sim j \\ Q_{ij} \le -1}} (-1 - \log(-Q_{ij})) + \sum_{\substack{i \not\sim j \\ Q_{ij} \ge 1}} (Q_{ij} - \log Q_{ij} - 1)\\ 
\text{subject to: } 
                   & \diag \bm{Q} = C  \\
                   &  \bm{Q} \succeq 0 
\end{alignedat}
\end{equation}
\end{lemma}

Using the complementary slackness conditions for the exponential cones, we obtain the following explicit relationships between $\bm{P}^*$ and $\bm{Q}^*$. 
\begin{proposition}\label{prop:algebraic-rel}
Let $\bm{P}^*$ and $\bm{Q}^*$ be corresponding optimal solutions to \cref{prob:creg} and \cref{prob:dualopt}, respectively. We have the following relationships between the off-diagonal entries of $\bm{P}^*$ and $\bm{Q}^*$: 
\begin{enumerate}
    \item If $i \sim j$ and $Q_{ij}^* \le -1$, then $Q_{ij}^*  = -\frac{1}{P_{ij}^*}$.
    \item If $i \sim j$ and $Q_{ij}^* > -1$, then $P_{ij}^* = 1$.
    \item If $i \not\sim j$ and $Q_{ij}^* < 1$, then $P_{ij}^* = 0$.
    \item If $i \not\sim j$ and $Q_{ij}^* \ge 1$, then $Q_{ij}^* = \frac{1}{1-P_{ij}^*}$. 
\end{enumerate}
\end{proposition}
\begin{proof}
Suppose $i \sim j$ and $Q_{ij}^* < -1$. The KKT conditions \cref{eq:kkt2,eq:kkt3} together with \cref{lemma:equiv} imply that 
\begin{equation}\label{eq:kkt}
P_{ij}^*Q_{ij}^* + 1 + \log(-Q_{ij}^*) + \alpha_{ij}^* = 0.
\end{equation}
Recall from the definition of \cref{prob:creg} that $\alpha_{ij} \le \log P_{ij}$; at optimality this inequality becomes an equality. Applying \cref{lemma:tech}, we see that \cref{eq:kkt} has the unique solution $P_{ij}^*Q_{ij}^* = -1$, so $Q_{ij}^* = -\frac{1}{P_{ij}^*}$ as desired. On the other hand, if $Q_{ij}^* \ge -1$, then the KKT conditions yield $P_{ij}^* = \log P_{ij}^* + 1$. Thus \cref{lemma:tech} also yields $P_{ij}^* = 1$. 

The other two statements for $i \not\sim j$ follow similarly. 
\end{proof}
These explicit relationships between $\bm{P}^*$ and $\bm{Q}^*$  lead to bounds on the entries of $\bm{P}^*$, including the following corollary. 
\begin{corollary}\label{corollary:bound}
Suppose $C \ge 1$. If $i \sim j$, then $P_{ij}^* \ge \frac{1}{C}$, and if $i \not\sim j$, then $P_{ij}^* \le 1 - \frac{1}{C}$.
\end{corollary}
\begin{proof}
Since $\bm{Q} \succeq 0$ and $\diag \bm{Q} = C$, it follows that $|Q_{ij}^*| \le C$. The conclusions then follow upon applying \cref{prop:algebraic-rel}.
\end{proof}

Note that $\bm{Q}$ can be decomposed as $\bm{Y}\bm{Y}^\top$, where $\bm{Y} \in \RR^{n \times k}$. Thus the rows of $\bm{Y}$ lie on the $(k-1)$-sphere with radius $\sqrt{C}$, and $Q_{ij}$ is the dot product between the $i$th and $j$th rows of $\bm{Y}$. By \cref{prop:algebraic-rel}, it follows that if $i \sim j$, the indices $(i,j)$ where $P_{ij}^* < 1$ correspond to the indices where $Q_{ij}^* \le -1$. Similarly, if $i \not\sim j$, the indices $(i,j)$ where $P_{ij}^* > 0$ correspond to the indices where $Q_{ij}^* \ge 1$.

\subsection{Connection to MAXCUT}\label{subsec:maxcut}
Inspecting \cref{prob:dualopt} and noting that $|Q_{ij}^*| \le C$, it follows for $C \le 1$ that we have the following problem with an objective that is homogeneous in $C$: 
\begin{equation}
\begin{alignedat}{3}
\min_{\bm{Q}}   &\quad&  \sum_{i \sim j} Q_{ij} &&          & \\
\text{subject to: } 
                   &\quad&  \diag \bm{Q} = C \le 1    && &  \\
                   &\quad&  \bm{Q} \succeq 0   && &
\end{alignedat}
\end{equation}
In fact, this is an instance of the SDP relaxation of \textsc{maxcut} with weight matrix equal to the adjacency matrix \cite{goemans1995improved}. This suggests that \cref{prob:creg} has deeper connections to classical combinatorial graph structure recovery problems. This also implies that the off-diagonal entries of the optimal $\bm{P}^*$ do not change in the region $C \le 1$. In particular, we recover $\bm{A}$ as in \cref{prop:exact}. 

\section{Main results}
\label{sec:main}
We now derive our major results on the behavior of optimal solutions to $\cref{prob:creg}$ and $\cref{prob:dual}$. Naturally, an optimal solution to the MAP problem will still tend towards larger probabilities where there are edges and smaller probabilities where there are nonedges. However, the penalty term also constrains the norms of the individual latent vectors.

These bounds on the inferred probability matrix $\bm{P}^*$ allow us to further study properties of the inference problem, such as consistency, in \cref{subsec:generalization}. Perhaps unsurprisingly, our bounds depend on node degrees and the regularization parameter $C$. In fact, we show that $P_{ij}^* = \Theta(1/C)$ for all $i$ and $j$, so the probabilities and latent vector lengths can be made arbitrarily small for sufficiently large $C$. These bounds follow solely from the complementary slackness conditions and basic facts about PSD matrices.
\begin{proposition}\label{prop:main-bound}
Suppose that $G = (V, E)$, and write $d_i \coloneqq \deg(i)$ for all $i \in V$. Then the following bounds hold.
\begin{alignat}{4}
\frac{1}{C\min_j d_j} &\le P_{ii}^* \le \frac{d_i}{C} \quad &&\forall i \label{eq:bd-Pii} \\
P_{ij}^* &\le \frac{\sqrt{d_id_j}}{C} \quad &&\forall i \neq j \label{eq:bd-Pij} \\
Q_{ij}^* &\le \min(d_i, d_j) - 1 \quad &&\forall i \neq j \label{eq:bd-Qij}
\end{alignat}
Additionally, if $C \ge \max_i d_i$, we have:
\begin{alignat}{4}
Q_{ij}^* &\le \frac{C}{C - \sqrt{d_id_j}} \quad &&\forall i \not\sim j \label{eq:bd-C-Qij} \\
P_{ii}^* &\ge \frac{d_i}{C} - \sum_{i \not\sim j} \frac{\sqrt{d_id_j}}{C(C-\sqrt{d_id_j})} \quad &&\forall i \label{eq:bd-C-Pii}
\end{alignat}
\end{proposition}
\begin{proof}
The complementary slackness condition $\bm{P}^*\bm{Q}^* = \bm{0}$ (see \cref{prop:dual}) furnishes us with $n^2$ equations. Focusing on the diagonal entries and applying \cref{prop:algebraic-rel} provides crude bounds on the entries of $\bm{P}^*$ and $\bm{Q}^*$. 

In particular, from $\diag(\bm{P}^*\bm{Q}^*) = \bm{0}$, we obtain for each $i$
\begin{equation}\label{eq:comp}
CP_{ii}^* + \sum_{i \sim j} \max(-1, Q_{ij}^*) + \sum_{i \not\sim j} \frac{P_{ij}^*}{1-P_{ij}^*} = 0.
\end{equation}
Since $\frac{P_{ij}^*}{1-P_{ij}^*} \ge 0$, and the sum over $i \sim j$ has $d_i$ terms, we have $P_{ii}^* \le \frac{d_i}{C}$,
which proves the second inequality of \cref{eq:bd-Pii}.
Since the first and third terms of \cref{eq:comp} are nonnegative, the sum over $i \sim j$ must be nonpositive, so $Q_{ij}^* \le d_i - 1$. \cref{eq:bd-Qij} follows by symmetry.

In a PSD matrix, every $2 \times 2$ submatrix is PSD, so we have 
\[
P_{ij}^* \le \sqrt{P_{ii}^*P_{jj}^*} \le \frac{\sqrt{d_id_j}}{C},
\]
which is \cref{eq:bd-Pij}. On the other hand, we have 
\[ P_{ii}^* \ge \max_j \frac{(P_{ij}^*)^2}{P_{jj}^*} \ge \max_j \frac{C(P_{ij}^*)^2}{d_j} \ge \frac{1}{C\min_j d_j},\]
where the last inequality follows from \cref{corollary:bound}. Hence the first inequality of \cref{eq:bd-Pii} holds.

Now suppose $C \ge \max d_i$. For \cref{eq:bd-C-Qij}, note that for $i \not\sim j$ we have $Q_{ij}^* \le \frac{1}{1-P_{ij}^*}$ by \cref{prop:algebraic-rel}. Because $C \ge \max d_i$, $P_{ij}^* \le 1$ by \cref{eq:bd-Pij}. Hence $Q_{ij}^* \le \frac{C}{C - \sqrt{d_id_j}}$ as desired.
Finally, substituting \cref{eq:bd-Pij} into \cref{eq:comp} and noting again that the denominator in $\cref{eq:comp}$ is nonnegative, we have 
\[
0 \le CP_{ii}^* - d_i + \sum_{i \not\sim j} \frac{\sqrt{d_id_j}}{C-\sqrt{d_id_j}},
\]
and rearranging yields \cref{eq:bd-C-Pii}.
\end{proof}

A couple points are worth noting. If node $i$ is isolated, i.e. $d_i = 0$, then the learned latent vector is the zero vector. More generally, the maximum size of $P_{ij}$ is controlled by the connectivity of nodes $i$ and $j$. Finally, we immediately obtain the following corollary characterizing the trace of $\bm{P}^*$. 
\begin{corollary}\label{corollary:trace}
Let $m = |E|$ be the number of edges in $G$. We have 
\begin{equation}
0 \le \tr \bm{P}^* \le \frac{2m}{C}. \label{eq:bd-trP}    
\end{equation}
Furthermore, if $C > n$ then 
\begin{equation}
\tr \bm{P}^* \ge \frac{2m}{C} - \frac{n^3}{C(C-n)}. \label{eq:bd-C-trP}
\end{equation}
\end{corollary}
\begin{proof}
Sum \cref{eq:bd-Pii} over $i$, using the fact that $\sum_i d_i = 2m$, to obtain \cref{eq:bd-trP}.
\cref{eq:bd-C-trP} follows from \cref{eq:bd-C-Pii}, noting that $\sqrt{d_id_j} \le n$. 
\end{proof}

\begin{remark}\label{remark:signless}
With $C=1$, as discussed in \cref{subsec:maxcut}, the dual cone program reduces to \textsc{maxcut}, and the corollary gives $\tr \bm{P}^* \le 2m$. In this case, we can construct an explicit optimal solution to \cref{prob:mle} that saturates the trace bound for $C=1$. Indeed, the signless Laplacian $\bm{L} \coloneqq \bm{D} + \bm{A}$, where $\bm{D} = \diag(d_i)$, is PSD, its off-diagonal entries are those of $\bm{A}$, and $\tr \bm{L} = 2m$ \cite{cvetkovic2007signless}. Also for $C \le 1$, the discussion in \cref{subsec:maxcut} implies $P^*_{ij} = A_{ij}$ for $i \neq j$.
\end{remark} 
\subsection{Rank of optimal solutions} Since our problem is aimed at recovering low rank solutions, it behooves us to study the rank of optimal solutions. We propose a strategy for analyzing the rank, partly motivated by our bounds on the trace of optimal solutions. We construct an example where $\rank \bm{P}^* = \frac{n}{2}$; hence, no \emph{deterministic} bound can show that $\rank \bm{P}^* < \frac{n}{2}$, although a \emph{probabilistic} rank bound may still hold. Furthermore, we establish that in the regime where $C > n$, the naive trace-rank inequality furnished by the convex envelope relationship can only show that $\rank \bm{P}^* \ge \frac{n}{2}$. Hence, the trace-rank inequality gives the best possible deterministic bound in that regime. In doing so, we narrow the possible hypotheses for a positive rank reduction theorem. 

The relevant optimality condition for studying the rank comes from the PSD constraint that $\bm{P}^*\bm{Q}^* = \bm{0}$. It follows that $\rank \bm{P}^* + \rank \bm{Q}^* \le \rank \bm{P}^* + \dim \Nul \bm{P}^* = n$. Therefore, we can upper-bound $\rank \bm {P}^*$ by lower-bounding $\rank \bm{Q}^*$. It is natural to formulate bounds using the trace, as the trace is the convex envelope of rank. More precisely, for any $\bm{M} \in \RR^{n \times n}$, $\rank \bm{M} \ge \tr \bm{M}/\norm{\bm{M}}$. Conveniently, $\tr \bm{Q}^* = nC$ by construction, so it remains to upper-bound the operator norm $\norm{\bm{Q}^*}$. To that end, we hope to apply the Gershgorin disks theorem.

However, we now show that for sufficiently large $C$, this method cannot show that $\rank \bm{Q}^* > \frac{n}{2}$. 
\begin{proposition}\label{prop:rank-barrier}
Let $\bm{Q}^*$ be an optimal solution to \cref{prob:dual}. Then for any $C > \max d_i$, the tightest bound offered by Gershgorin is $\rank \bm{Q}^* \ge \frac{n}{2}$. Furthermore, this bound is tight.
\end{proposition}
\begin{proof}
If $C > \max d_i$, then \cref{prop:main-bound} implies that $P_{ij}^* < 1$ for all $i \neq j$. Therefore for $i \sim j$, \cref{prop:algebraic-rel} implies that $|Q_{ij}^*| \ge \frac{C}{\sqrt{d_id_j}}$. Even if we assume that $Q_{ij}^* = 0$ for all $i \not\sim j$ and $|Q_{ij}^*| = \frac{C}{\sqrt{d_id_j}}$ for all $i \sim j$, Gershgorin gives
\[
\rank \bm{Q}^* \ge \frac{nC}{C + \max d_i \cdot \frac{C}{\min d_i}} \ge \frac{n}{2}. 
\]

To see that this bound is tight, consider the graph consisting of $\frac{n}{2}$ disconnected 2-cliques. From \cref{prop:main-bound}, $P_{ii}^* = \frac{1}{C}$ for all $i$. Similarly, $P_{ij}^* = \frac{1}{C}$ for $i \sim j$, and \cref{eq:comp} implies that $P_{ij}^* = 0$ for $i \not\sim j$. Then the rank of $\bm{P}^*$ is clearly $\frac{n}{2}$. 
\end{proof}

\begin{remark}\label{remark:rank-barrier}
The phenomenon of vanishing entries of $\bm{P}^*$ posing an obstruction to rank results can be generalized as follows. For $0 \le p \le 1$, let $\Lambda(p)$ be the set of pairs $(i, j)$ with $i \neq j$ such that $P_{ij}^* = p$. Form the complement of the induced subgraph of $\Lambda(0)$. For any positive $k$, if this complement graph contains a $k$-clique, then $\rank \bm{P}^* \ge k$, as there are at least $k$ mutually orthogonal latent vectors. In \cref{subsec:likelihood,subsec:generalization} we will see how vanishing entries are also troublesome for consistency results.
\end{remark}

\subsection{Likelihood bounds}\label{subsec:likelihood}
Armed with the results of previous sections, we hope to study consistency results. We argue that for \cref{prob:creg}, it is natural to formulate consistency results in terms of the Kullback-Leibler (KL) divergence. However, we find that the KL divergence is not well behaved. Instead, we are able to show that the (closely related) difference in Bernoulli log likelihoods of the true $\bm{P}$ and the inferred $\bm{P}^*$ is $O(m\log C)$ with high probability.

To begin the discussion, we define the following notation and link the Bernoulli likelihood function to the KL divergence away from the diagonal.
\begin{definition}
The Bernoulli likelihood given an adjacency matrix $\bm{A}$ is defined as 
\[
L_{\bm{A}}(\bm{M}) = \sum_{i \neq j} A_{ij}\log(M_{ij}) + (1-A_{ij})\log(1-M_{ij}).
\]
This is precisely the objective function in \cref{prob:mle}. 
We denote the regularized objective function by
\[
\mathscr{L}_{\bm{A}, C}(\bm{M}) \coloneqq L_{\bm{A}}(\bm{M}) - C\tr(\bm{M}).
\]
\end{definition}
\begin{definition}
We define the \emph{matrix KL-divergence} between $\bm{X}$ and $\bm{Y} \in \RR^{n \times n}$ to be the mean off-diagonal entrywise divergence---that is,
\[
D(\bm{X} \parallel \bm{Y}) = -\frac{1}{n(n-1)} \sum_{i \neq j} \left[ X_{ij} \log\left(\frac{Y_{ij}}{X_{ij}}\right) + (1-X_{ij}) \log\left(\frac{1-Y_{ij}}{1-X_{ij}}\right)\right]. 
\]
\end{definition}

We define the quantity only on the off-diagonal entries because the diagonal entries of $\bm{P}$ do not carry probabilistic meaning. For a fixed $\bm{M}$, the definitions imply that 
\begin{equation} D(\bm{P} \parallel \bm{M}) = -\frac{1}{n(n-1)} \EE[L_{\bm{A}}(\bm{P}) - L_{\bm{A}}(\bm{M})], \end{equation}
where the expectation is taken with respect to the random adjacency matrix $\bm{A}$ distributed according to $\bm{P}$. All subsequent expectations, unless noted otherwise, are also taken with respect to $\bm{A}$.

Let $\bm{A}_0$ be sampled i.i.d. from $\bm{P}$. Denote by $\bm{P}^*_{\bm{A}_0, C}$ the optimal solution to \cref{prob:creg} with realization $\bm{A}_0$ and hyperparameter $C$. We hope to develop an upper bound on $D(\bm{P} \parallel \bm{P}^*_{\bm{A}_0, C})$. We first claim that it is impossible to obtain a finite deterministic bound on this quantity. Suppose that $0 < P_{ij} < 1$ for all $i \neq j$, so that $\bm{A}_0$ is the empty graph with positive probability. Assuming $\bm{A}_0$ is indeed the empty graph and $\bm{A}$ is nonempty, it follows that $L_{\bm{A}}(\bm{P}^*_{\bm{A}_0. C}) = -\infty$. However, we can still hope for a probabilistic bound. For example, if one could identify conditions on $\bm{P}$ under which $\frac{1}{n} \le P_{ij} \le 1 - \frac{1}{n}$ with high probability, then the results of \cref{subsec:generalization} apply.

Nevertheless, it is clear that the likelihood function of \cref{prob:creg} merits its own study. We present our main result in this direction below; the details of the proof are provided in \cref{app:b}. The upshot is that with high probability, the likelihood difference $L_{\bm{A}}(\bm{P}) - L_{\bm{A}}(\bm{P}^*)$ is $O(m\log C)$.

\begin{theorem}\label{thm:likelihood-bound}
Let $\bm{A} \sim \bm{P} \in \RR^{n \times n}$ and let $C \ge 1$. Then we have
\begin{equation}\label{ineq:kl-bound}
L_{\bm{A}}(\bm{P}) - L_{\bm{A}}(\bm{P}^*) \le C\tr\bm{P} - C\tr\bm{P}^*.
\end{equation}
Furthermore, for $\eta > 0$ and $C > \max d_i$, we have
\begin{equation}\label{ineq:likelihood-lower}
    L_{\bm{A}}(\bm{P}) - L_{\bm{A}}(\bm{P}^*) \ge -\eta - H[\bm{P}] + 2m\log C - \sum_i d_i\log d_i,
\end{equation}
and 
\begin{equation}\label{ineq:likelihood-upper}
    L_{\bm{A}}(\bm{P}) - L_{\bm{A}}(\bm{P}^*) \le \eta - H[\bm{P}] + 2m\log C + 2m,
\end{equation}
with probability at least $1 - 4\exp(-\eta^2/(8n(n-1) + 4\eta))$.
Here, we have defined $H[\bm{P}] \coloneqq -\sum_{i \neq j} P_{ij}\log P_{ij}$, the entropy of the edge probability matrix.
\end{theorem}
\begin{proof}[Proof sketch]
Since $\bm{P}$ is feasible for the convex program \cref{prob:creg},
\begin{align}\label{ineq:kl}
0 &\le \mathscr{L}_{\bm{A}, C}(\bm{P}^*) - \mathscr{L}_{\bm{A}, C}(\bm{P}) \\
&\le L_{\bm{A}}(\bm{P}^*) - L_{\bm{A}}(\bm{P}) + C\tr \bm{P} - C \tr \bm{P}^*.
\end{align}
Rearranging yields \cref{ineq:kl-bound}. To show \cref{ineq:likelihood-lower,ineq:likelihood-upper}, we first use a recent Bernoulli concentration result \cite[Corollary 1]{zhao2020note} to bound $L_{\bm{A}}(\bm{P})$ with high probability. Then we bound $L_A(\bm{P}^*)$ by constructing primal and dual feasible points and applying \cref{corollary:trace}.
\end{proof}

From here, we can apply Hoeffding's inequality to shift the bound's dependence on $\bm{A}$ to dependence on $\bm{P}$ with high probability, yielding the following corollary.
\begin{corollary}\label{corollary:hoeffding-reduction}
Suppose $M \coloneqq \sum_{i < j} P_{ij} \in \omega(n)$. Then with high probability, 
$L_{\bm{A}}(\bm{P}) - L_{\bm{A}}(\bm{P}^*) = O(M\log C)$.
\end{corollary}
\subsection{Consistency guarantees for a modified program}\label{subsec:generalization}
We now turn to formulating a consistency bound after slightly modifying \cref{prob:creg}. In particular, we impose the mild assumption that $\frac{1}{n} \le P_{ij} \le 1 - \frac{1}{n}$ for $i \neq j$. We think of these as nontriviality conditions---they ensure that no node has expected degree less than $1$ or greater than $n-1$. Under this model, we obtain that with high probability, the KL divergence $D(\bm{P} \parallel \bm{P}^*)$ is $O(\EE[\tr \bm{P}^*]n^{-1/2}\log n + Cn^{-2}\tr \bm{P})$. We also identify conditions under which this bound is effective.

With these new constraints, the primal problem becomes
\begin{equation}\tag{REG-MOD}\label{prob:reg-mod}
\begin{alignedat}{3}
\max_{\bm{P}}              &\quad&  \sum_{i \sim j} \log P_{ij} + \sum_{i \not\sim j} \log(1-P_{ij}) - C\tr \bm{P} &&          & \\
\text{subject to: } &\quad&  \frac{1}{n} \le P_{ij} \le 1 - \frac{1}{n}    && &\quad \text{ for all } i \neq j\\
                    &\quad&  \bm{P} \succeq 0  &&   &\quad 
\end{alignedat}
\end{equation}

With the additional constraints in \cref{prob:reg-mod}, a duality analysis produces the following bound (cf. \cref{corollary:trace}); see \cref{app:b} for the complete proof. To state the result we recall the $\Lambda$ notation introduced in \cref{remark:rank-barrier}. We extend the notation in a few  ways. First, we let $\Lambda((p, q))$ denote the set of indices $i \neq j$ where $p < P_{ij}^* < q$. Next, we define $\Lambda_0(p) \coloneqq \Lambda(p) \cap E^c$ and $\Lambda_1(p) \coloneqq \Lambda(p) \cap E$, extending the definition to $\Lambda((p,q))$ in the obvious way. Finally, we define $Z(\bm{P}^*) \coloneqq \#\Lambda_0(\tfrac{1}{n}) - \#\Lambda_0((\tfrac{1}{n}, 1-\tfrac{1}{n}))$. 
\begin{proposition}\label{prop:new-trace-bound}
Let $\bm{P}^*$ be an optimal solution to \cref{prob:reg-mod}. Then 
\begin{equation}\label{ineq:new-trace-bound}
\tr \bm{P}^* \le \frac{2m}{C} + \frac{1}{n}\#\Lambda(\tfrac{1}{n})- \#\Lambda_1(\tfrac{1}{n})- \frac{n-1}{C}\#\Lambda_0(1-\tfrac{1}{n}) - \frac{1}{n-1}\#\Lambda_0((\tfrac{1}{n}, 1-\tfrac{1}{n})).
\end{equation}
It follows that 
\begin{equation}\label{ineq:new-trace-bound-simp}
\tr \bm{P}^* \le \frac{2m}{C} + \frac{1}{n}Z(\bm{P}^*).
\end{equation}
\end{proposition}

We assume that the true $\bm{P}$ is feasible for \cref{prob:reg-mod}, so that $\mathscr{L}_{\bm{A}, C}(\bm{P}^*) \ge \mathscr{L}_{\bm{A}, C}(\bm{P})$. For clarity, we denote by $W_{\bm{A}}(\bm{M})$ the centered random variable $L_{\bm{A}}(\bm{M}) - \EE[L_{\bm{A}}(\bm{M})]$. Recalling \cref{ineq:kl}, but this time adding and subtracting expectations, we find that 
\begin{align*}
0 &\le \EE[L_{\bm{A}}(\bm{P}) - L_{\bm{A}}(\bm{P^*})]  + W_{\bm{A}}(\bm{P}^*) - W_{\bm{A}}(\bm{P}) - C\tr(\bm{P}^*) + C\tr(\bm{P}) \\
&\le -n(n-1)D(\bm{P} \parallel \bm{P}^*) + W_{\bm{A}}(\bm{P}^*) - W_{\bm{A}}(\bm{P}) - C\tr(\bm{P}^*) + C\tr(\bm{P}).
\end{align*}
\Cref{thm:bern-concen} gives us concentration of $W_{\bm{A}}(\bm{P})$, so it suffices to obtain a probabilistic upper bound on $W_{\bm{A}}(\bm{P}^*)$. We bound this quantity over a compact neighborhood of $\bm{P}^*$ rather than directly bounding it for $\bm{P}^*$. In particular, define the set $F_{\bm{A}}$ to be the intersection of the feasible set of \cref{prob:reg-mod} with the trace ball centered at the origin with radius $\tr \bm{P}^*$. Clearly $\bm{P}^* \in F_{\bm{A}}$, and since \cref{prop:new-trace-bound} controls $\tr \bm{P}^*$, $F_{\bm{A}}$ is a compact subset of the original feasible set. Thus it suffices to probabilistically upper-bound $\sup_{\bm{M} \in F_{\bm{A}}} W_{\bm{A}}(\bm{M})$.

Indeed, this can be achieved through a similar approach to that in \cite{davenport20141}; the details of the proof are left to \cref{app:c}. The upshot is that for sparse enough graphs, we indeed obtain consistency as the number of nodes in the RDPG increases.

\begin{theorem}\label{thm:generalization}
Let $\bm{P}^*$ be an optimal solution to \cref{prob:reg-mod}, and suppose $\bm{P}$ is feasible for \cref{prob:reg-mod}. Then there is an absolute constant $K$, such that with probability at least $1-\delta$, we have 
\[
D(\bm{P} \parallel \bm{P}^*) \le \frac{K\EE[\tr(\bm{P}^*)]\log n}{\delta \sqrt{n}} + \frac{C\tr(\bm{P}) - C\tr(\bm{P}^*)}{n(n-1)} - \frac{L_{\bm{A}}(\bm{P})+H[\bm{P}]}{n(n-1)}.
\]
\end{theorem}

A quick application of \cref{prop:new-trace-bound} and \cref{thm:bern-concen} yields the following corollary. 
\begin{corollary}\label{corollary:generalization}
With high probability, we have 
\[
D(\bm{P} \parallel \bm{P}^*) \le \frac{KM\log n}{\delta C\sqrt{n}} + \frac{K\EE[Z(\bm{P}^*)]\log n}{\delta n\sqrt{n}} + \frac{C\tr(\bm{P})}{n(n-1)} + O(1/n),
\]
where $M \coloneqq \EE[m] = \sum_{i < j} P_{ij}$ is the expected number of observed edges.
\end{corollary}
\begin{remark}
There are two sources of the probabilistic nature of the bound: $\delta$ and the concentration of $L_{\bm{A}}(\bm{P})$. We can for example take $\delta = O(\log n)$ and $\eta = \omega(n)$ in \cref{thm:bern-concen}.
\end{remark}
We discuss the shape of the bound and its implications. Note that \emph{a priori} 
\[
D(P_{ij} \parallel P_{ij}^*) \le \log n,
\]
since for every feasible $\bm{M}$ we have $\frac{1}{n} \le M_{ij} \le 1-\frac{1}{n}$. Further inspection of \cref{corollary:generalization} reveals that under certain conditions on $\bm{P}^*$ and $\bm{P}$, \cref{thm:generalization} nontrivially improves on this $\log n$ bound and gives consistency as $n \to \infty$.
\begin{corollary}
Suppose $\EE[Z(\bm{P}^*)] = O(n^{3/2}/\log n)$. Then the optimal $C$ with respect to the bound in \cref{corollary:generalization} satisfies $C^2 = O(Mn^{3/2}\log n/\tr \bm{P})$. If we further have $\tr \bm{P} = O(M/n)$ and $M = o(n^{7/4}/(\log n)^{1/2})$, then $D(\bm{P} \parallel \bm{P}^*) = o(1)$ with high probability.
\end{corollary}
\begin{remark}
Note that owing to \cref{lemma:expected-edges}, the assumption on $\tr \bm{P}$ is the tightest possible one. The exact assumption on $\EE[Z(\bm{P}^*)]$ can potentially be relaxed if the proof of \cref{prop:new-trace-bound} is refined. 
\end{remark}

\section{Experimental results}
\label{sec:experiments}
We investigate the performance of the inference technique for a variety of tasks. We find that the learned embedding appears to be pushed towards lower and lower ranks as the regularization hyperparameter $C$ is increased. Motivated by this, we examine regularization profiles that demonstrate that \cref{prob:creg} improves on the ASE in the spectral norm distance. This naturally leads to a cross validation procedure to select an embedding dimension. Through visualization and cluster analysis, we also find that our method recovers latent geometry better than the ASE in synthetic and more realistic networks. 

The two primary quantities of interest for a solution $\bm{P}^*$ to \cref{prob:creg} are (1) its fit to the true $\bm{P}$ and (2) its rank, which we denote by $d^*$. The former, which measures consistency, is difficult to measure directly with real data because we do not usually know $\bm{P}$. Similarly, we might not know the true dimension $d$ of the underlying latent vector distribution, so we must balance other considerations to select a value of the regularization parameter $C$. As a first step towards understanding how $C$ affects $d^*$, we present the spectra in \cref{fig:eigvals}, which illustrates an inverse relationship between $C$ and $d^*$. This makes sense since \cref{prob:creg} is an MAP problem where $C$ controls the strength of the Laplace prior on the spectrum of $\bm{P}$. In subsequent sections we also observe this phenomenon of approximate low rank solutions as we increase $C$; a similar phenomenon has been studied in related matrix completion problems \cite{recht2010guaranteed}. 

To examine how well our inference procedure recovers the rank and geometry of the true $\bm{P}$, we turn to synthetic data, sampled from an RDPG with a \emph{known} latent vector distribution. In \cref{subsec:synthetic}, we consider three classes of synthetic data: the SBM with a fixed block probability matrix, latent vector distributions supported on $S^{d-1}$, and latent vector distributions supported on the unit ball in $\RR^d$. 


\subsection{Implementation details}
The convex cone program \cref{prob:creg} was implemented in CVXPY using the MOSEK \cite{mosek} and SCS \cite{scs} solvers. Both solvers support semidefinite and exponential cone constraints, but differ in their algorithms: MOSEK uses an interior point method, while SCS uses the alternating direction method of multipliers (ADMM). In general, we found that MOSEK converges to more accurate solutions than SCS. On the other hand, as SCS uses the first-order ADMM, it can handle larger problem instances.

Given $\bm{P}$, we measured the closeness of fit between $\bm{P}^*$ and $\bm{P}$ by directly computing $\norm{\bm{P}^* - \bm{P}}$. In order to compute $d^*$, we used the \texttt{eigvalsh} function in the \texttt{scipy.linalg} package, with a tolerance of $1 \cdot 10^{-6}$ for MOSEK and $1 \cdot 10^{-3}$ for SCS. These thresholds were selected based on empirical observations of eigenvalue spectra for $\bm{P}^*$ such as the one shown in \cref{fig:eigvals}, and reflect a disparity in numerical precision between the solvers.

\begin{figure}
\centering
\includegraphics[width=0.5\textwidth]{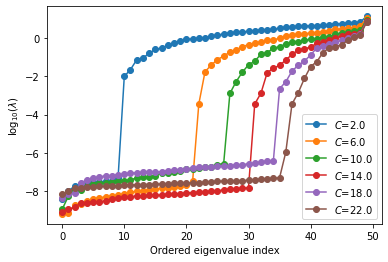}
\caption{Sample eigenvalue plots for $\bm{P}^*_C$ for a fixed realization $\bm{A}$ sampled from \cref{fig:balls-latent-vectors} as $C$ is increased. The eigenvalues are sorted in ascending order and their base-ten logarithms are plotted. Note the abrupt drop in eigenvalue size to around $10^{-6}$ for each $C$.}
\label{fig:eigvals}
\end{figure}

\subsection{Synthetic data experiments}\label{subsec:synthetic}
In order to evaluate the consistency of \cref{prob:creg}, we conducted the following three experiments, all using MOSEK for conic optimization. In each experiment, we sampled $50$ latent vectors from a fixed distribution and formed a probability matrix $\bm{P}$. Keeping $\bm{P}$ fixed, we then generated $100$ RDPGs, and for each RDPG we ran the inference problem for $2 \le C \le 25$ with a step size for $1$. We computed spectral distance $\norm{\bm{P}_C^* - \bm{P}}$ and $d^*_C$ for each RDPG and each value of $C$.

\begin{figure}[tb]
\begin{subfigure}{0.45\textwidth}
    \centering
    \includegraphics[width=0.95\linewidth]{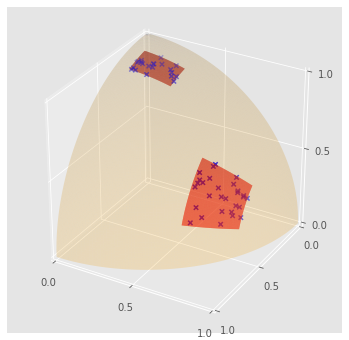}
    \caption{Latent vectors supported on $S^2$.}
    \label{fig:patch-latent-vectors}
\end{subfigure}%
\begin{subfigure}{0.45\textwidth}
    \centering
    \includegraphics[width=0.95\linewidth]{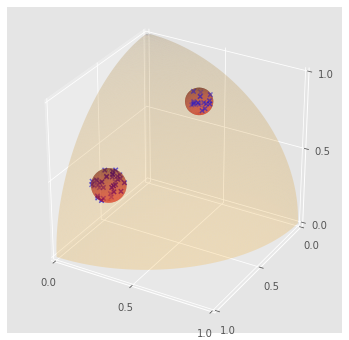}
    \caption{Latent vectors supported on two balls.}
    \label{fig:balls-latent-vectors}
\end{subfigure}
\caption{Ground truth latent vector positions. We sampled 50 latent vectors in both figures; their positions are marked by blue X's. The orange surface traces out $S^2$, and the red surfaces trace out the support of the respective latent vector distributions.}
\end{figure}

In \cref{fig:patch-cross-val}, we illustrate the results for the two patches on $S^2$ shown in \cref{fig:patch-latent-vectors}. 
\begin{figure}[tb]
\centering
\newcommand{\imgwidth}{0.33\textwidth}
\begin{subfigure}{\imgwidth}
    \centering
    \includegraphics[width=0.9\linewidth]{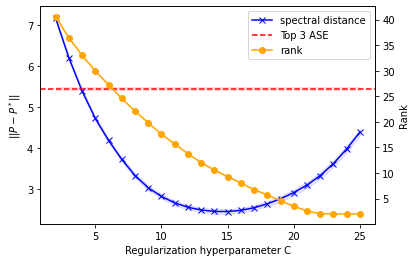}
    \caption{Patches on $S^2$}
    \label{fig:patch-cross-val}
\end{subfigure}\hfill%
\begin{subfigure}{\imgwidth}
    \centering
    \includegraphics[width=0.9\linewidth]{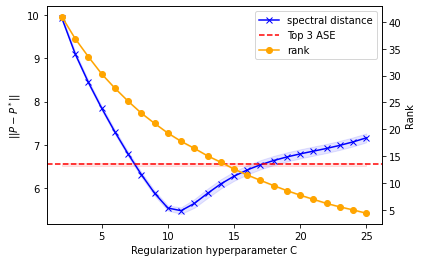}
    \caption{Balls in $B^3$}
    \label{fig:balls-cross-val}
\end{subfigure}\hfill%
\begin{subfigure}{\imgwidth}
    \centering
    \includegraphics[width=0.9\linewidth]{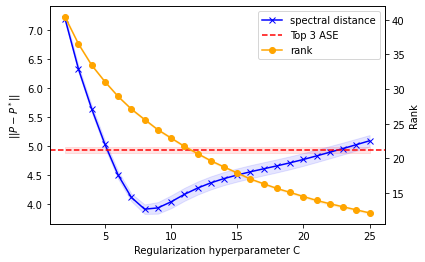}
    \caption{SBM}
    \label{fig:sbm-cross-val}
\end{subfigure}
\caption{Regularization profiles for synthetic 50 node RDPGs. The rank (yellow) decreases monotonically. Note also the clear local minimum in spectral norm distance (blue) in all three plots. Furthermore, for a certain range of $C$, our method improves on the top-$3$ eigenvector ASE spectral norm distance.} 
\end{figure}
Interestingly, although the spectral norm distance does not achieve its local minimum where $d_C^* = 3$, we will later see that the 3D and 2D visualizations are still salient. We also point out the smoothness of the cross validation curve and the clear local minimum for the spectral norm at $C \approx 13$, with $d_C^* \approx 10$, corresponding to a significant dimensionality reduction from the original rank 50 solution. 

In \cref{fig:balls-cross-val}, we show the results from the two balls in the nonnegative unit ball depicted in \cref{fig:balls-latent-vectors}. We chose this example because the latent vectors do not all have the same Euclidean norm, which could reasonably be interpreted as a disparity in the intrinsic popularities of nodes in the RDPG. In \cref{fig:balls-cross-val}, we see that the light blue shaded band representing the two--standard deviation interval is wider. Indeed, since the intra-cluster probabilities are not close to 1 (unlike the previous experiment), there is higher variance in the realizations of the RDPG. Nevertheless, we observe a local minimum at $C \approx 10$ for the cross validation curve and a monotonic relationship between $d^*$ and $C$. 

In \cref{fig:sbm-cross-val}, we depict the results from an SBM with blocks of respective size 15, 10, and 25, using the block probability matrix $\bigl[\begin{smallmatrix}
0.25 & 0.05 & 0.02 \\
0.05 & 0.35 & 0.07 \\
0.02 & 0.07 & 0.40
\end{smallmatrix}\bigr]$. Although the SBM had three clusters (as the block probability matrix is $3 \times 3$), the cross validation graph is visually similar to that in \cref{fig:balls-cross-val}, where the ground truth only contained two clusters. 

In \cref{fig:patch-cross-val,fig:balls-cross-val,fig:sbm-cross-val}, it is apparent that our method beats the ASE in terms of spectral norm distance, even when the ASE is given oracle access to $d$. Even if $\bm{P}$ is unknown, the same cross validation procedure can be used by simply replacing $\bm{P}$ with $\bm{A}$ in the metric. However, we found experimentally that using $\norm{\bm{A} - \bm{P}^*}$ does not perform as well as using $\norm{\bm{A}^2 - (\bm{P}^*)^2}$ for cross validation purposes. Hence we computed the latter metric for the inferred solutions in the earlier experiments to obtain the curves in \cref{fig:squared-patch,fig:squared-balls,fig:squared-sbm}. Interestingly, the shapes of the two curves are very similar, especially in \cref{fig:squared-balls,fig:squared-sbm}. More importantly, they have similar local minima as a function of $C$. Thus, it is reasonable to use this cross validation technique to select the best $C$ and thus single out a preferred embedding or rank. 

\begin{figure}
\centering
\newcommand{\imgwidth}{0.33\textwidth}
\begin{subfigure}{\imgwidth}
    \centering
    \includegraphics[width=0.9\linewidth]{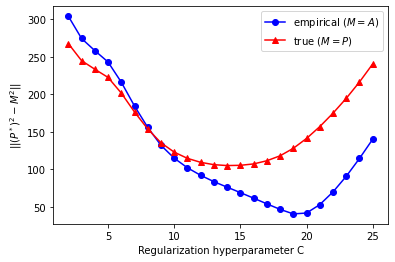}
    \caption{Patches on $S^2$}
    \label{fig:squared-patch}
\end{subfigure}\hfill%
\begin{subfigure}{\imgwidth}
    \centering
    \includegraphics[width=0.9\linewidth]{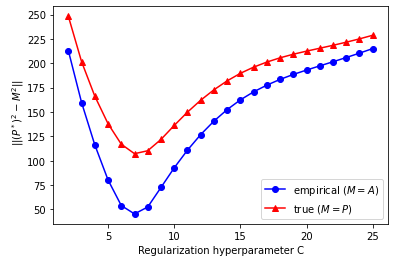}
    \caption{Balls in $B^3$}
    \label{fig:squared-balls}
\end{subfigure}\hfill%
\begin{subfigure}{\imgwidth}
    \centering
    \includegraphics[width=0.9\linewidth]{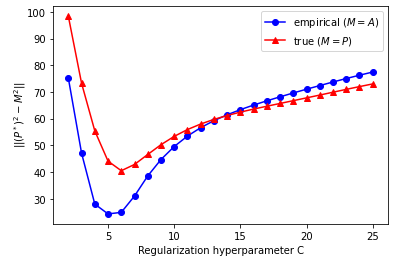}
    \caption{SBM}
    \label{fig:squared-sbm}
\end{subfigure}
\caption{Cross validation graphs comparing $\norm{\bm{P}^2 - (\bm{P}^*)^2}$ (true) and $\norm{\bm{A}^2 - (\bm{P}^*)^2}$ (empirical). The shape of the curves are generally similar. Crucially, the locations of the local minimum of the empirical and true curves are similar as a function of $C$.}
\end{figure}
\subsection{Visualization through latent vector extraction}\label{subsec:geometry-recovery}
Given the solution $\bm{P}_C^*$ for some $C$ in \cref{prob:dual}, we can extract an embedding $\bm{X}$ by taking the top 2 or 3 eigenvectors and scaling them by the square roots of their eigenvalues. We have seen that selecting $C$ via cross validation can recover better solutions $\bm{P}_C^*$ measured by the spectral norm; the hope is that the same effect can also be observed visually. 

We conducted the following experiment to demonstrate the improvement in visualization capabilities in 3 dimensions over the ASE and the applications of the inference technique to larger graphs. Starting from the synthetic distribution illustrated in \cref{fig:patch-latent-vectors}, we generated an RDPG with 300 nodes; note the two clusters on $S^2$ that support the distribution. Using SCS, we computed the regularization profile and cross validation plot for $10 \le C \le 150$ with a step size of 10 (shown in \cref{fig:large-patch-profile,fig:large-patch-cross-val}) and selected the best $C$ in spectral norm distance. We then visually compared the top-3 eigenvector embedding $\bm{P}_C^*$ with the top-3 eigenvector ASE. \Cref{fig:patch-ase} demonstrates that the ASE, while approximately recovering the geometry of the latent vector distribution, is more diffuse than the corresponding visualization in \cref{fig:patch-top-3-cross-val}. In particular, notice in \cref{fig:patch-top-3-cross-val} that fewer of the embeddings lie outside of $S^2$, and that the geometry more closely matches that displayed in \cref{fig:patch-latent-vectors}. 

\begin{figure}[tb]
\begin{subfigure}{0.45\textwidth}
    \centering
    \includegraphics[width=0.95\linewidth]{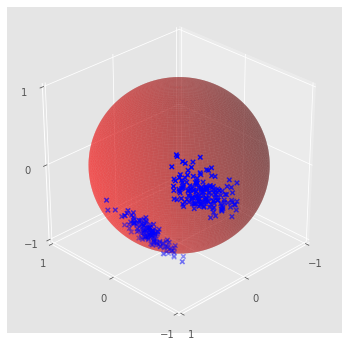}
    \caption{ASE visualization}
    \label{fig:patch-ase}
\end{subfigure}%
\begin{subfigure}{0.45\textwidth}
    \centering
    \includegraphics[width=0.95\linewidth]{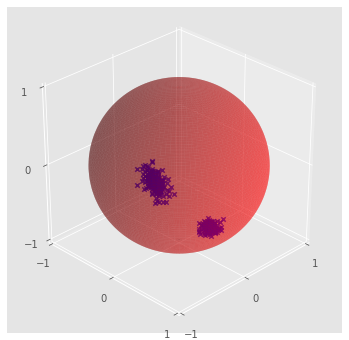}
    \caption{Embedding for $C = 60$}
    \label{fig:patch-top-3-cross-val}
\end{subfigure}\\
\begin{subfigure}{0.45\textwidth}
    \centering
    \includegraphics[width=0.95\linewidth]{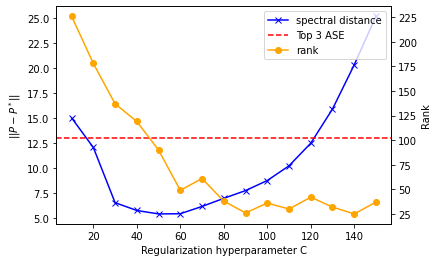}
    \caption{Regularization profile}
    \label{fig:large-patch-profile}
\end{subfigure}%
\begin{subfigure}{0.45\textwidth}
    \centering
    \includegraphics[width=0.95\linewidth]{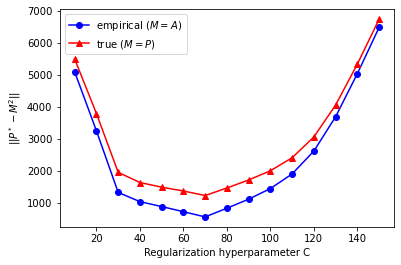}
    \caption{Cross validation}
    \label{fig:large-patch-cross-val}
\end{subfigure}
\caption{Top-3 eigenvector visualizations with the ASE (left) and our method (right). The unit sphere is overlayed, and the points are rotated to enhance visual understanding for the embedding.}
\end{figure}

We also demonstrate the clustering effects of increasing regularization. Using the same 300 node RDPG, we compare the top-3 eigenvector visualizations for various values of $C$ in \cref{fig:embedding1,fig:embedding2,fig:embedding3}. The shape of the embedding remains relatively consistent as $C$ is increased. As $C$ is increased, the numerical rank of $\bm{P}^*_C$ decreases, and the two clusters become more compact. 

In order to quantify this notion, we use the Dunn index \cite{dunn1973fuzzy}; larger Dunn indices correspond to more compact clusters. 
In \cref{fig:clustering-plot} we see that as $C$ is increased, the Dunn index also increases. Notice that the ASE has a low Dunn index, and that as $C$ increases, the learned embeddings approach a Dunn index close to that of the true embedding. Furthermore, the best embedding (around $C=60$) determined in \cref{fig:large-patch-profile,fig:large-patch-cross-val} has a similar Dunn index to the true embedding.



\begin{figure}[tb]
\centering
\begin{subfigure}{0.4\textwidth}
    \centering
    \includegraphics[width=0.7\linewidth]{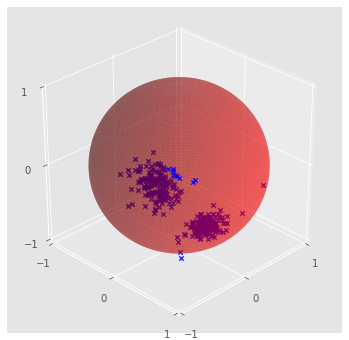}
    \caption{$C = 20$ and $d^* = 136$}
    \label{fig:embedding1}
\end{subfigure}%
\begin{subfigure}{0.4\textwidth}
    \centering
    \includegraphics[width=0.7\linewidth]{pics/patch_2_c_60.png}
    \caption{$C=60$ and $d^* = 60$}
    \label{fig:embedding2}
\end{subfigure}\\
\begin{subfigure}{0.4\textwidth}
    \centering
    \includegraphics[width=0.7\linewidth]{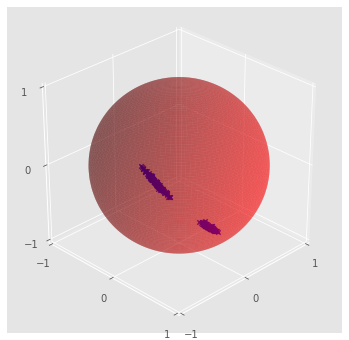}
    \caption{$C=100$ and $d^*= 20$}
    \label{fig:embedding3}
\end{subfigure}%
\begin{subfigure}{0.4\textwidth}
    \centering
    \includegraphics[width=0.95\linewidth]{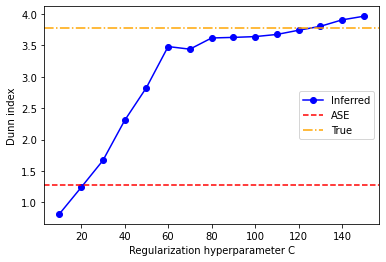}
    \caption{Dunn indices}
    \label{fig:clustering-plot}
\end{subfigure}
\caption{Three dimensional spectral embeddings of the same graph with $n=300$. Note the tighter clusters as $C$ increases, in accordance with the increasing Dunn indices.}
\end{figure}

\subsection{Realistic network experiments}\label{subsec:real-network}

We first study embeddings of the Karate Club Graph. We solved the inference problem for $10 \le C \le 100$ with a step size of $5$. In \cref{fig:karate-club-mle}, we select the smallest $C$ with $d_C^* = 2$ and compare the latent vector visualization to that of the top-2 eigenvector ASE. The color of each node represents which community it belongs to. Comparing the two embeddings, we note that the positions of the communities are swapped, which illustrates the inherent nonidentifiability of the embedding. The shrinkage effects of increasing regularization are also apparent. Although both the ASE and our method make the communities linearly separable, our method tightens the clusters more than the ASE does. 

\cref{fig:senate-graph} depicts the results of an experiment using synthetic data generated from the voting patterns of the 114th U.S. Senate. We choose this session as all $n=100$ senators were active throughout the entire session. 
Similar latent space models \cite{pauls, poole_paper} and network analyses of community detection from voting data have been studied in \cite{deford_old, moody_mucha_2013, Mucha2, Mucha876, porter, waugh2011party}. 
The edge probability $P_{ij}$ is defined to be the fraction of votes in which $i$ and $j$ voted the same way. If this vote agreement data fits the RDPG assumption, our inference procedure should be able to recover two well defined clusters corresponding to the Democratic and Republican political parties. Indeed, \cref{fig:senate-graph-mle} reveals an inferred embedding that displays the two clusters prominently. Moreover, the nodes in the middle represent the more moderate senators, suggesting that the model recovers accurate detailed geometry.

\begin{figure}[tb]
\centering
\begin{subfigure}{0.45\textwidth}
    \centering
    \includegraphics[width=0.95\linewidth]{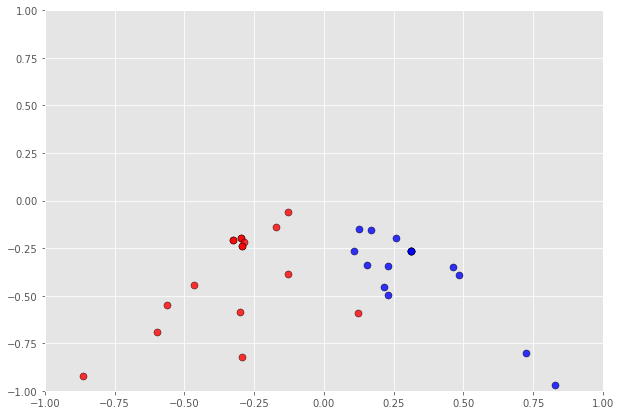}
    \caption{ASE}\label{fig:karate-club-ase}
\end{subfigure}%
\begin{subfigure}{0.45\textwidth}
    \centering
    \includegraphics[width=0.95\linewidth]{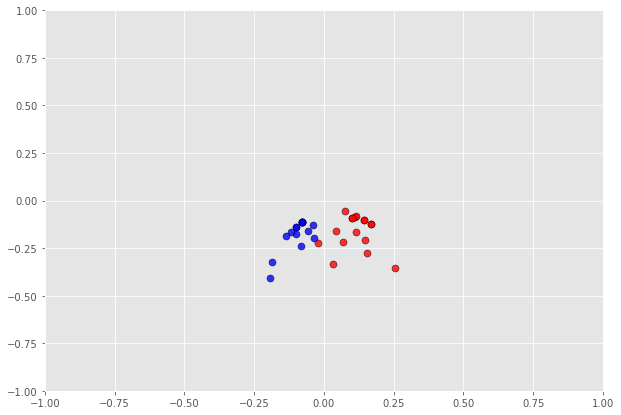}
    \caption{$C=80$ and $d^* = 2$.}\label{fig:karate-club-mle}
\end{subfigure}
\caption{Top-2 eigenvector visualization for the Karate Club Graph using the ASE (left) and the rank 2 solution for $C = 80$ (right). The axis limits are fixed to demonstrate the shrinkage effects of the regularization.}
\end{figure}

\begin{figure}[tb]
\centering
\begin{subfigure}{0.45\textwidth}
    \centering
    \includegraphics[width=0.95\linewidth]{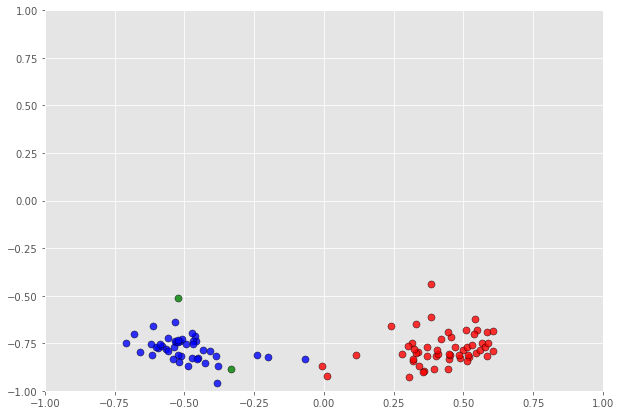}
    \caption{ASE} \label{fig:senate-graph-ase}
\end{subfigure}%
\begin{subfigure}{0.45\textwidth}
    \centering
    \includegraphics[width=0.95\linewidth]{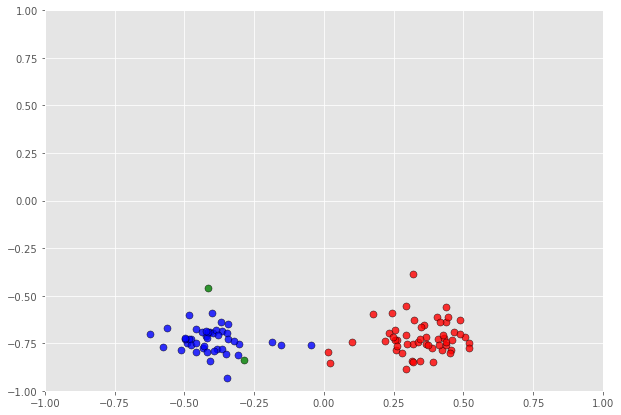}
    \caption{$C = 40$ and $d^* = 2$}
    \label{fig:senate-graph-mle}
\end{subfigure}
\caption{Top-2 eigenvector visualization for a synthetic U.S. Senate Vote Graph using the ASE (left) and the rank 2 solution for $C = 40$ (right). Blue nodes denote Democratic senators, red nodes denote Republican senators, and green nodes denote independent senators. The axis limits are fixed so the embeddings can be directly compared.}
\label{fig:senate-graph}
\end{figure}


\section{Conclusions}
\label{sec:conclusions}
We have presented a conic programming approach to MAP inference of latent vector embeddings of RDPGs. Our inference technique offers a different perspective on recovering RDPG latent geometry compared to existing methods which typically center around spectral decompositions of graph matrices. The conic formulation led to a surprisingly explicit relationship between the primal and dual problems. Importantly, the exponential cone constraints yield an exact algebraic relationship between the optimal solutions to the primal and dual problems, whereas the semidefinite cone constraint and trace regularization lead to bounds on the optimal solutions. 

Some of the limitations of our analysis stem from a lack of nontrivial lower bounds on $P_{ij}^*$. For example, while we've shown the impossibility of a useful deterministic bound on rank for large $C$, a better handle on elementwise lower bounds would likely produce a nontrivial rank reduction result for small $C$. Similarly, lower bounds would enable a probabilistic KL consistency result that does not require the model to be altered. In particular, if one is able to prove that $\frac{1}{n} \le P_{ij}^* \le  1-\frac{1}{n}$, then the results of \cref{subsec:generalization} immediately follow. 

The experimental results show that it is often possible to recover solutions of arbitrary rank by varying the regularization parameter $C$. It is not possible to do so for all realizations, as established by \cref{prop:rank-barrier}. Nevertheless, we conjecture that with high probability, the rank of inferred embeddings can be reduced nontrivially, say to order $O(\log n)$. Although we were not able to prove spectral norm consistency of the extracted embedding, the results of our cross validation and visualization experiments suggest nontrivial improvements on the ASE over a certain bounded range of $C$ values. Furthermore, we observe that as $C$ increases arbitrarily, the spectral norm distance worsens, which agrees with the theoretical guarantee that the entries of $\bm{P}^*$ are $\Theta(1/C)$. 

The proposed conic optimization problem is solved via second-order methods using MOSEK and via the first-order ADMM method using SCS. For the algorithm to be of interest for larger graphs, there would likely need to be a first-order implementation that could exploit low-rank complexity reductions. In this vein, an extension of the B\"urer-Monteiro method \cite{cifuentes2019burer}, along with more theoretical guarantees of optimality, would be particularly interesting.

\appendix
\section{Derivation of dual program} \label{app:a}
\begin{proof}[Proof of \cref{prop:dual}]
Form the Lagrangian of $\cref{prob:creg}$, with notation defined as in \cref{sec:dual}. Swapping the $\min$ and $\max$ operators and eliminating the primal constraints by enforcing the dual constraints, we obtain
\begin{align*}
 \min_{\substack{\bm{Q} \succeq 0 \\ \bm{\lambda}_{ij}, \bm{\nu}_{ij} \in K_{\mathrm{exp}}^*}} \max_{\alpha,\beta,\bm{P}} \Bigg[ &\sum_{i \sim j} \alpha_{ij} + \sum_{i \not\sim j} \beta_{ij}  - C\tr(\bm{P})  \\
&+ \sum_{i \neq j} \bm{\lambda}_{ij} \cdot (P_{ij}, 1, \alpha_{ij}) + \sum_{i \neq j} \bm{\nu}_{ij} \cdot (1-P_{ij}, 1, \beta_{ij}) + \langle \bm{Q}, \bm{P}\rangle\Bigg].
\end{align*}
Collecting terms involving each primal variable, we obtain
\begin{align*}\label{eq:exp}
\min_{\substack{\bm{Q} \succeq 0 \\ \bm{\lambda}_{ij}, \bm{\nu}_{ij} \in K_{\mathrm{exp}}^*}} \max_{\alpha,\beta,\bm{P}} \quad & 
\sum_{i \sim j} [(1 + t_{ij})\alpha_{ij} + w_{ij}\beta_{ij}] + \sum_{i \not\sim j} [(1+w_{ij})\beta_{ij} + t_{ij}\alpha_{ij}] \\ 
& + \sum_{i} (Q_{ii}-C)P_{ii} + \sum_{i \neq j} [(Q_{ij} + r_{ij} - u_{ij})P_{ij} + s_{ij} + u_{ij} + v_{ij}],
\end{align*} 
where $\bm{\lambda}_{ij} = (r_{ij}, s_{ij}, t_{ij})$ and $\bm{\nu}_{ij} = (u_{ij}, v_{ij}, w_{ij})$.

This is in turn equivalent to the following dual program:
\begin{equation}\label{eq:dual-unsimplified}
\begin{alignedat}{3}
\min              &\quad&  \sum_{i \neq j} (s_{ij} + u_{ij} + v_{ij}) &&          & \\
\text{subject to: } &\quad&  t_{ij} = -1, w_{ij} = 0  && &\quad \text{ for all } i \sim j  \\
                   &\quad&  t_{ij} = 0, w_{ij} = -1  && &\quad \text{ for all } i \not\sim j  \\
                   &\quad& Q_{ij} = u_{ij} - r_{ij}  && &\quad \text{ for all } i \neq j  \\
                   &\quad&  \bm{\lambda}_{ij}, \bm{\nu}_{ij} \in K_{\exp}^*    && &\quad \text{ for all } i \neq j  \\
                   &\quad&  \diag \bm{Q} = C    && &  \\
                   &\quad&  \bm{Q} \succeq 0   && &
\end{alignedat}
\end{equation}


These explicit constraints follow from Lagrangian duality, or alternatively by noting that the outer minimization necessarily avoids inner objective values of $+\infty$. 

Recall from \cref{def:exp} that
\[ K_{\mathrm{exp}}^* \coloneqq \{(u, v, w) \in \RR^3 : u \ge -w\exp(v/w - 1), u > 0, w < 0\} \cup \{(u, v, 0) : u, v \ge 0\}. \]
Applying the explicit constraints in \cref{eq:dual-unsimplified}, we simplify the dual objective function as follows:
\begin{itemize}
\item For $i\sim j$, $\bm{\lambda}_{ij} \in K_{\mathrm{exp}}^*$ simplifies to $s_{ij} \ge -1 - \log r_{ij}$. Also, $w_{ij} = 0$ implies that $v_{ij} \ge 0$. Since there are no other constraints on $v_{ij}$, and since we are minimizing $v_{ij}$, it follows that at optimality $ v_{ij}^* = 0$ for $i \sim j$. Hence we can rewrite the $\sum_{i \neq j} v_{ij}$ as $\sum_{i \not\sim j} v_{ij}$. 

\item For $i \not\sim j$, the constraint for $\bm{\nu}_{ij}$ simplifies to $v_{ij} \ge -1 - \log u_{ij}$. Similarly, $t_{ij} = 0$ implies that $s_{ij} \ge 0$, and at optimality $s_{ij}^* = 0$. Hence we can rewrite $\sum_{i \neq j} s_{ij}$ as $\sum_{i \sim j} s_{ij}$. Putting these steps all together, the objective function becomes 
\[
\min \sum_{i \sim j} s_{ij} + \sum_{i\not\sim j} v_{ij} + \sum_{i \neq j} u_{ij}.
\]
\end{itemize}
Thus we obtain \cref{prob:dual}, as desired.
\end{proof}

We now prove the following simple lemma which is useful in the proof of \cref{prop:algebraic-rel} and \cref{lemma:equiv}.
\begin{lemma}\label{lemma:tech}
Let $x$ be a positive real number. If $x + \log x = 1$ or $x - \log x = 1$, then $x = 1$. Similarly, let $x$ be a negative real number. If $x + \log(-x) = -1$ or $x - \log(-x) = -1$, then $x=-1$. 
\end{lemma}
\begin{proof}
First suppose $x > 0$. By inspection, $x = 1$ satisfies $x \pm \log x = 1$. Since $x + \log x$ is monotonic, $x=1$ is the unique solution. Since $x - \log x$ is strictly convex and has a critical point at $x=1$, $x=1$ is the unique solution. The case of $x < 0$ follows similarly.
\end{proof}
\begin{proof}[Proof of \cref{lemma:equiv}]
Note that $s_{ij}$ is constrained by a single inequality $s_{ij} \ge -1 - \log r_{ij}$ and is being minimized in the objective function. Thus, at optimality, the inequality becomes tight, and we have $s_{ij} = -1 - \log r_{ij}$. By a similar argument, $v_{ij} = -1 - \log u_{ij}$ at optimality. These considerations reduce the dual to the following form, equivalent at optimality:
\begin{equation}
\begin{alignedat}{3}
\min              &\quad&  \sum_{i \sim j} (u_{ij}-1-\log r_{ij}) + \sum_{i \not\sim j} (u_{ij}-1-\log u_{ij}) &&          & \\
\text{subject to: } &\quad&  u_{ij} - r_{ij} = Q_{ij}     && &\quad \text{ for all } i \neq j \\
                    &\quad&  u_{ij}, r_{ij} \ge 0    && &\quad \text{ for all } i \neq j  \\
                   &\quad&  \diag \bm{Q} = C    && &  \\
                   &\quad&  \bm{Q} \succeq 0   && &
\end{alignedat}
\end{equation}
We simplify the dual further by considering each pair $i \ne j$ individually:
\begin{itemize}
\item Suppose $i \sim j$. Substituting $u_{ij} - Q_{ij}$ for $r_{ij}$, the corresponding objective term is $-1 - \log(u_{ij} - Q_{ij}) + u_{ij}$.
If $Q_{ij} \ge -1$, then by \cref{lemma:tech} this has a unique minimum at $u_{ij} = Q_{ij} + 1 \ge 0$, which is in the feasible region. Thus in this case $r_{ij} = 1$, and the objective term reduces to $Q_{ij}$. 

On the other hand, if $Q_{ij} < -1$, then due to domain constraints and monotonicity of the objective, the minimizer is $u_{ij} = 0$, so $r_{ij} = -Q_{ij}$. Thus, in this case the objective term is $-1 - \log(-Q_{ij})$. 

\item Next, consider the term $u_{ij} - 1 - \log u_{ij}$ where $i \not\sim j$.
If $Q_{ij} \le 1$, then this has a unique minimum at $u_{ij} = 1$ by \cref{lemma:tech}. Furthermore, the constraint $r_{ij} = u_{ij} - Q_{ij} \ge 0$ is satisfied. These terms vanish in the objective function. However, if $Q_{ij} \ge 1$, then the optimal choice for $u_{ij}$ is $u_{ij} = Q_{ij}$. In this case $r_{ij} = 0$, which simplifies the objective term to $Q_{ij} - \log Q_{ij} - 1$, thus proving the claim.
\end{itemize}
\end{proof}

\section{Proofs of likelihood bounds}\label{app:b}
We begin by restating the recent Bernoulli concentration result from \cite{zhao2020note}.

\begin{theorem}[{\cite[Corollary 1]{zhao2020note}}]\label{thm:bern-concen}
For $p_{ij} \in [0, 1]$, with $1 \le i \neq j \le n$, and all $\eta > 0$, we have 
\[
\PP\Bigg[\Big|L_{\bm{A}}(\bm{P}) - \EE[L_{\bm{A}}(\bm{P})]\Big| \ge \eta \Bigg] \le 4\exp\left(-\frac{\eta^2}{4(2n(n-1)+\eta)}\right).
\]
\end{theorem}

\begin{proof}[Proof of \cref{thm:likelihood-bound}]
We construct an upper bound by way of a dual-feasible solution to \cref{prob:dualopt} in the restricted domain of symmetric \emph{diagonally dominant} matrices, which are manifestly PSD. We construct a dual-feasible, diagonally dominant $\bm{Q}$ as follows.
\begin{equation}\label{eq:dual-construction}
    Q_{ij} = \begin{cases}
    C & i = j \\
    0 & i \not\sim j \\
    -C/d_i & i \sim j
    \end{cases}
\end{equation}

In fact, for $C > \max(d_i)$, \cref{eq:dual-construction} is optimal for symmetric, diagonally dominant $\bm{Q}$. To see this, note the following facts. 
\begin{itemize}
    \item The objective function of \cref{prob:dualopt} is monotonically nondecreasing in $Q_{ij}$.
    \item When we impose diagonal dominance, the program decouples for each $i$.
    \item The non-edge contribution to the objective function is nonnegative and vanishes for $Q_{ij} \le 1$.
\end{itemize}
Using these facts, the restricted program decouples to the following for each $i$:
\begin{equation}\tag{DCREG-RES}\label{prob:dual-res}
\begin{alignedat}{3}
\min \quad             & \sum_{i\sim j} -1 - \log(-Q_{ij})  \\
\text{subject to: } 
                  &  -C \le Q_{ij} \le -1 & i \sim j \\
                  & -\sum_{j} Q_{ij} \le C & \text{ for all } i
\end{alignedat}
\end{equation}
The assumption that $C > \max(d_i)$ ensures that the feasible set is nonempty. In fact, \cref{prob:dual-res} can be analytically minimized using AM-GM to see that indeed we should take $Q_{ij} = -\frac{C}{d_i}$.

Plugging \cref{eq:dual-construction} back into the original objective function in \cref{prob:dualopt}, we obtain 
\[
\sum_{i \sim j} -1 - \log\left(\frac{C}{d_i}\right) = -2m - \sum_i d_i\log\left(\frac{C}{d_i}\right).
\]

Thus $\mathscr{L}_{\bm{A}, C}(\bm{P}^*) \le -2m - \sum_i d_i\log(\frac{C}{d_i})$, and combining this with \cref{corollary:trace}, we obtain 
\begin{equation}\label{ineq:dual-feasible}
L_{\bm{A}}(\bm{P}^*) \le -\sum_i d_i\log\left(\frac{C}{d_i}\right) = -2m\log C + \sum_i d_i\log d_i.
\end{equation}

Now write
\[
L_{\bm{A}}(\bm{P}) - L_{\bm{A}}(\bm{P}^*) = (L_{\bm{A}}(\bm{P}) - \EE[L_{\bm{A}}(\bm{P})]) + (\EE[L_{\bm{A}}(\bm{P})] - L_{\bm{A}}(\bm{P}^*)).
\]
We can combine \cref{thm:bern-concen} with \cref{ineq:dual-feasible} to obtain \cref{ineq:likelihood-lower}, as desired.

For the upper bound on $L_{\bm{A}}(\bm{P}) - L_{\bm{A}}(\bm{P}^*)$, note that the scaled signless Laplacian $\frac{1}{C}(\bm{D} + \bm{A})$ in \cref{remark:signless} achieves objective value $-2m\log C - 2m$ for \cref{prob:creg}. Hence applying \cref{thm:bern-concen} again we obtain \cref{ineq:likelihood-upper}.
\end{proof}

\section{Proofs of consistency results}\label{app:c}
\begin{lemma}\label{lemma:expected-edges}
For any RDPG generated by edge probability matrix $\bm{P}$, we have
\begin{equation}
    \EE[m] = \sum_{i < j} P_{ij} \le \frac{n-1}{2}\tr \bm{P}.
\end{equation}
\end{lemma}
\begin{proof}
The first equality immediately follows from the generative model. For the inequality, note that since $\bm{P}$ is PSD, we have $\frac{P_{ii} + P_{jj}}{2} \ge P_{ij}$. Summing over all $i \neq j$, we obtain $(n-1)\tr \bm{P}\ge \sum_{i \neq j} P_{ij}$, and the result follows by symmetry of $\bm{P}$.
\end{proof}

Now we take the dual of $\cref{prob:reg-mod}$ in hopes of obtaining a similar trace bound to \cref{corollary:trace}. We introduce dual variables $\bm{Q} \succeq 0$ and $c_{ij}, d_{ij} \ge 0$ for $i \neq j$. Taking the Lagrangian and swapping $\min$ and $\max$ as before, we obtain 
\begin{align*}
 \min_{\substack{\bm{Q} \succeq 0 \\ c_{ij}, d_{ij} \ge 0}} \max_{\bm{P}} \Bigg[ &\sum_{i \sim j} \log P_{ij} + \sum_{i \not\sim j} \log(1-P_{ij}) - C\tr(\bm{P})  \\
&+ \sum_{i \neq j} c_{ij}\left(P_{ij}-\frac{1}{n}\right) + d_{ij}\left(1-\frac{1}{n}-P_{ij}\right) +  \langle \bm{Q}, \bm{P}\rangle\Bigg].
\end{align*}

The KKT conditions require the gradient of the Lagrangian with respect to $\bm{P}$ to vanish at optimality. We do casework on $i$ and $j$.
\begin{itemize}
    \item If $i \sim j$, then we obtain $-\frac{1}{P_{ij}} = c_{ij} - d_{ij} + Q_{ij} $.
    \item If $i \not\sim j$, then we obtain
    $\frac{1}{1-P_{ij}} = c_{ij} - d_{ij} + Q_{ij}$.
    \item If $i=j$, then we obtain $Q_{ii} = C$. 
\end{itemize}
Hence in the dual problem, $\bm{Q} \succeq 0$ and $\diag \bm{Q} = C$, from which it follows that $|Q_{ij}| \le C$. The KKT conditions also require that $c_{ij}^*(P_{ij}^* - \frac{1}{n}) = 0$,  $d_{ij}^*(1 - \frac{1}{n} - P_{ij}^*) = 0$, and $\bm{P}^*\bm{Q}^* = 0$. Considering the diagonal entries of the latter constraint as before (see \cref{prop:main-bound}) we obtain
\begin{equation}\label{ineq:new-cslackness}
CP_{ii}^* + \sum_{i \sim j} P_{ij}^*Q_{ij}^* + \sum_{i\not\sim j} P_{ij}^*Q_{ij}^* = 0. 
\end{equation}

We now bound the edge terms by way of the KKT conditions.
\begin{itemize}
    \item If $\frac{1}{n} < P_{ij}^* \le 1-\frac{1}{n}$, then $Q_{ij}^* \ge -\frac{1}{P_{ij}^*}$. Hence $P_{ij}^*Q_{ij}^* \ge -1$.
    \item If $P_{ij}^* = \frac{1}{n}$, then recalling that $Q_{ij}^* \ge -C$, we have $P_{ij}^*Q_{ij}^* \ge -\frac{C}{n}$.
\end{itemize}

Before combining these two bounds, we define $N(i) \coloneqq \{j: i \sim j\}$ and $N(i)^c \coloneqq \{j: i \not\sim j\}$. We have
\begin{equation}\label{ineq:new-trace-edge}
\sum_{N(i)} P_{ij}^*Q_{ij}^* \ge -d_i + \sum_{\Lambda(\frac{1}{n}) \cap N(i)} \left(1 -  \frac{C}{n}\right).
\end{equation}

We now turn to bounding the nonedge terms.
\begin{itemize}
    \item If $\frac{1}{n} < P_{ij}^* < 1-\frac{1}{n}$, then $Q_{ij}^* = \frac{1}{1-P_{ij}^*}$. Hence $P_{ij}^*Q_{ij}^* > \frac{1}{n-1}$.
    \item If $P_{ij}^* = 1 - \frac{1}{n}$, then $Q_{ij}^* \ge n$. Hence $P_{ij}^*Q_{ij}^* \ge n-1$.
    \item If $P_{ij}^* = \frac{1}{n}$, then again recalling that $Q_{ij}^* \ge -C$, we have $P_{ij}^*Q_{ij}^* \ge -\frac{C}{n}$.
\end{itemize}

Hence 
\begin{equation}\label{ineq:new-trace-nonedge}
\sum_{N(i)^c} P_{ij}^*Q_{ij}^* \ge \sum_{\Lambda((\frac{1}{n}, 1-\frac{1}{n})) \cap N(i)^c} \frac{1}{n-1} + \sum_{\Lambda(1-\frac{1}{n}) \cap N(i)^c} (n-1) - \sum_{\Lambda(\frac{1}{n}) \cap N(i)^c} \frac{C}{n}.
\end{equation}

Having done all of the legwork, now we are ready to quickly prove \cref{prop:new-trace-bound}.
\begin{proof}[Proof of \cref{prop:new-trace-bound}]
Plug \cref{ineq:new-trace-edge} and \cref{ineq:new-trace-nonedge} into \cref{ineq:new-cslackness}. Summing over $i$ and dividing by $C$, we obtain \cref{ineq:new-trace-bound}. \Cref{ineq:new-trace-bound-simp} immediately follows since $\#\Lambda \ge 0$. 
\end{proof}

\begin{proof}[Proof of \cref{thm:generalization}]
Much of the following proof takes inspiration from \cite{davenport20141}, although there are a few key modifications that must be made to adapt their technique. The first key observation is that $\log x$ and $\log(1-x)$ are $n$-Lipschitz on the interval $[1/n, 1-1/n]$. In order to apply the contraction principle \cite[Theorem 4.12]{ledoux2013probability}, under typical hypotheses we need a Lipschitz function $\varphi$ which vanishes at $0$. However, upon closer examination of its proof, it suffices for our purposes to show that $|\log x| \le \rho x$ for some $\rho$ and all $x \in [1/n, 1-1/n]$. It is not hard to see that the minimum such $\rho$ is $n\log n$. Applying symmetrization \cite[Section 6.4]{vershynin2018high} by introducing symmetric Bernoulli variables $\epsilon_{ij}$, followed by contraction \cite{ledoux2013probability}, we arrive at 
\begin{align*}
\EE\left[\sup_{\bm{M} \in F_{\bm{A}}} W_{\bm{A}}(\bm{M})\right] & \le \EE_{\bm{A}, \epsilon}\left[\sup_{\bm{M} \in F_{\bm{A}}} \sum_{i \neq j} \epsilon_{ij}(\1_{i\sim j}\log(M_{ij}) + \1_{i\not\sim j}\log(1-M_{ij}))\right] \\
&\le n\log n \cdot \EE_{\bm{A}, \epsilon} \left[\sup_{\bm{M} \in F_{\bm{A}}} \sum_{i \neq j} \epsilon_{ij} (\1_{i \sim j} M_{ij} + \1_{i \not\sim j} M_{ij})\right] \\
&\le n\log n \cdot \EE_{\bm{A}, \epsilon} \left[\sup_{\bm{M} \in F_{\bm{A}}}|\langle\bm{E}, \bm{M}\rangle|\right].
\end{align*}
In the third line we have defined the random matrix $\bm{E}$ by $E_{ij} \coloneqq \epsilon_{ij}$ for $i \neq j$ and $E_{ii} \coloneqq 0$, and we have used the fact that $\1_{i \sim j} + \1_{i\not\sim j} = \1_{i \neq j}$. Now we have $|\langle\bm{E}, \bm{M}\rangle| \le \norm{\bm{E}} \tr(\bm{M})$, since $\bm{M}$ is PSD. By \cite[Theorem 1.1]{seginer2000expected} we have that there is some absolute constant $K$ so that $\norm{\bm{E}} \le K\sqrt{n}$. 
We thus have that
\begin{align*}
\EE\left[\sup_{\bm{M} \in F_{\bm{A}}} W_{\bm{A}}(\bm{M})\right] &\le  n\log n \cdot \EE_{\epsilon}[\norm{\bm{E}}] \EE_{\bm{A}}\left[\sup_{\bm{M} \in F_{\bm{A}}} \tr(\bm{M})\right] \\
& \le Kn^{3/2} \log n \EE_{\bm{A}}[\tr(\bm{P}^*)],
\end{align*}
where we have used the definition of $F_{\bm{A}}$ in the last line. 
Dividing through by $n(n-1)$, and applying Markov's inequality with parameter $\delta$ we arrive at the desired inequality.
\end{proof}
\section*{Acknowledgments}
We would like to thank Justin Solomon and Chris Scarvelis for their thoughtful comments. We would also like to thank Nicolas Boumal and Aude Genevay for helpful discussions. 
\bibliographystyle{siamplain}
\bibliography{references}

\begin{thebibliography}{10}

\bibitem{abbe2017community}
{\sc E.~Abbe}, {\em Community detection and stochastic block models: recent
  developments}, The Journal of Machine Learning Research, 18 (2017),
  pp.~6446--6531.

\bibitem{agrawal2018rewriting}
{\sc A.~Agrawal, R.~Verschueren, S.~Diamond, and S.~Boyd}, {\em A rewriting
  system for convex optimization problems}, Journal of Control and Decision, 5
  (2018), pp.~42--60.

\bibitem{agterberg2020vertex}
{\sc J.~Agterberg, Y.~Park, J.~Larson, C.~White, C.~E. Priebe, V.~Lyzinski,
  et~al.}, {\em Vertex nomination, consistent estimation, and adversarial
  modification}, Electronic Journal of Statistics, 14 (2020), pp.~3230--3267.

\bibitem{mosek}
{\sc M.~ApS}, {\em The MOSEK Modeling Cookbook. Version 3.2.2B}, 2020,
  \url{https://docs.mosek.com/modeling-cookbook/expo.html}.

\bibitem{arroyo2018maximum}
{\sc J.~Arroyo, D.~L. Sussman, C.~E. Priebe, and V.~Lyzinski}, {\em Maximum
  likelihood estimation and graph matching in errorfully observed networks},
  arXiv preprint arXiv:1812.10519,  (2018).

\bibitem{athreya2017statistical}
{\sc A.~Athreya, D.~E. Fishkind, M.~Tang, C.~E. Priebe, Y.~Park, J.~T.
  Vogelstein, K.~Levin, V.~Lyzinski, and Y.~Qin}, {\em Statistical inference on
  random dot product graphs: a survey}, The Journal of Machine Learning
  Research, 18 (2017), pp.~8393--8484.

\bibitem{athreya2016limit}
{\sc A.~Athreya, C.~E. Priebe, M.~Tang, V.~Lyzinski, D.~J. Marchette, and D.~L.
  Sussman}, {\em A limit theorem for scaled eigenvectors of random dot product
  graphs}, Sankhya A, 78 (2016), pp.~1--18.

\bibitem{boumal2016non}
{\sc N.~Boumal, V.~Voroninski, and A.~Bandeira}, {\em The non-convex
  burer-monteiro approach works on smooth semidefinite programs}, in Advances
  in Neural Information Processing Systems, 2016, pp.~2757--2765.

\bibitem{boyd_vandenberghe_2004}
{\sc S.~Boyd and L.~Vandenberghe}, {\em Convex Optimization}, Cambridge
  University Press, 2004, \url{https://doi.org/10.1017/CBO9780511804441}.

\bibitem{burer2003nonlinear}
{\sc S.~Burer and R.~D. Monteiro}, {\em A nonlinear programming algorithm for
  solving semidefinite programs via low-rank factorization}, Mathematical
  Programming, 95 (2003), pp.~329--357.

\bibitem{cai2013max}
{\sc T.~Cai and W.-X. Zhou}, {\em A max-norm constrained minimization approach
  to 1-bit matrix completion}, The Journal of Machine Learning Research, 14
  (2013), pp.~3619--3647.

\bibitem{chen2016robust}
{\sc L.~Chen, C.~Shen, J.~T. Vogelstein, and C.~E. Priebe}, {\em Robust vertex
  classification}, IEEE Transactions on Pattern Analysis and Machine
  Intelligence, 38 (2016), pp.~578--590.

\bibitem{choi2012stochastic}
{\sc D.~S. Choi, P.~J. Wolfe, and E.~M. Airoldi}, {\em Stochastic blockmodels
  with a growing number of classes}, Biometrika, 99 (2012), pp.~273--284.

\bibitem{cifuentes2019burer}
{\sc D.~Cifuentes}, {\em On the burer-monteiro method for general semidefinite
  programs}, arXiv preprint arXiv:1904.07147,  (2019).

\bibitem{collins2002generalization}
{\sc M.~Collins, S.~Dasgupta, and R.~E. Schapire}, {\em A generalization of
  principal components analysis to the exponential family}, in Advances in
  neural information processing systems, 2002, pp.~617--624.

\bibitem{cvetkovic2007signless}
{\sc D.~Cvetkovi{\'c}, P.~Rowlinson, and S.~K. Simi{\'c}}, {\em Signless
  laplacians of finite graphs}, Linear Algebra and its applications, 423
  (2007), pp.~155--171.

\bibitem{davenport20141}
{\sc M.~A. Davenport, Y.~Plan, E.~Van Den~Berg, and M.~Wootters}, {\em 1-bit
  matrix completion}, Information and Inference: A Journal of the IMA, 3
  (2014), pp.~189--223.

\bibitem{davenport2016overview}
{\sc M.~A. Davenport and J.~Romberg}, {\em An overview of low-rank matrix
  recovery from incomplete observations}, IEEE Journal of Selected Topics in
  Signal Processing, 10 (2016), pp.~608--622.

\bibitem{deford_old}
{\sc D.~DeFord and D.~Rockmore}, {\em A random dot product model for weighted
  networks}, arXiv preprint arXiv:1611.02530,  (2016).

\bibitem{diamond2016cvxpy}
{\sc S.~Diamond and S.~Boyd}, {\em {CVXPY}: {A} {P}ython-embedded modeling
  language for convex optimization}, Journal of Machine Learning Research, 17
  (2016), pp.~1--5.

\bibitem{dunn1973fuzzy}
{\sc J.~C. Dunn}, {\em A fuzzy relative of the isodata process and its use in
  detecting compact well-separated clusters}, Journal of Cybernetics,  (1973).

\bibitem{goemans1995improved}
{\sc M.~X. Goemans and D.~P. Williamson}, {\em Improved approximation
  algorithms for maximum cut and satisfiability problems using semidefinite
  programming}, Journal of the ACM (JACM), 42 (1995), pp.~1115--1145.

\bibitem{landgraf2015dimensionality}
{\sc A.~J. Landgraf and Y.~Lee}, {\em Dimensionality reduction for binary data
  through the projection of natural parameters}, arXiv preprint
  arXiv:1510.06112,  (2015).

\bibitem{ledoux2013probability}
{\sc M.~Ledoux and M.~Talagrand}, {\em Probability in Banach Spaces:
  isoperimetry and processes}, Springer Science \& Business Media, 2013.

\bibitem{moody_mucha_2013}
{\sc J.~MOODY and P.~J. MUCHA}, {\em Portrait of political party polarization},
  Network Science, 1 (2013), p.~119–121.

\bibitem{Mucha2}
{\sc P.~J. Mucha and M.~A. Porter}, {\em Communities in multislice voting
  networks}, Chaos: An Interdisciplinary Journal of Nonlinear Science, 20
  (2010).

\bibitem{Mucha876}
{\sc P.~J. Mucha, T.~Richardson, K.~Macon, M.~A. Porter, and J.-P. Onnela},
  {\em Community structure in time-dependent, multiscale, and multiplex
  networks}, Science, 328 (2010), pp.~876--878.

\bibitem{o2015maximum}
{\sc L.~O'Connor, M.~M{\'e}dard, and S.~Feizi}, {\em Maximum likelihood latent
  space embedding of logistic random dot product graphs}, arXiv preprint
  arXiv:1510.00850,  (2015).

\bibitem{scs}
{\sc B.~O'Donoghue, E.~Chu, N.~Parikh, and S.~Boyd}, {\em {SCS}: Splitting
  conic solver, version 2.1.2}.
\newblock \url{https://github.com/cvxgrp/scs}, Nov. 2019.

\bibitem{pauls}
{\sc S.~D. Pauls, G.~Leibon, and D.~Rockmore}, {\em The social identity voting
  model: Ideology and community structures}, Research \& Politics,  (2015).

\bibitem{poole_paper}
{\sc K.~T. Poole and H.~Rosenthal}, {\em A spatial model for legislative roll
  call analysis}, American Journal of Political Science, 29 (1985),
  pp.~357--384, \url{http://www.jstor.org/stable/2111172}.

\bibitem{porter}
{\sc M.~A. Porter, P.~J. Mucha, M.~Newman, and A.~Friend}, {\em Community
  structure in the united states house of representatives}, Physica A:
  Statistical Mechanics and its Applications, 386 (2007), pp.~414 -- 438.

\bibitem{recht2010guaranteed}
{\sc B.~Recht, M.~Fazel, and P.~A. Parrilo}, {\em Guaranteed minimum-rank
  solutions of linear matrix equations via nuclear norm minimization}, SIAM
  review, 52 (2010), pp.~471--501.

\bibitem{seginer2000expected}
{\sc Y.~Seginer}, {\em The expected norm of random matrices}, Combinatorics,
  Probability and Computing, 9 (2000), pp.~149--166.

\bibitem{singer2011angular}
{\sc A.~Singer}, {\em Angular synchronization by eigenvectors and semidefinite
  programming}, Applied and computational harmonic analysis, 30 (2011),
  pp.~20--36.

\bibitem{sussman2014consistent}
{\sc D.~L. {Sussman}, M.~{Tang}, and C.~E. {Priebe}}, {\em Consistent latent
  position estimation and vertex classification for random dot product graphs},
  IEEE Transactions on Pattern Analysis and Machine Intelligence, 36 (2014),
  pp.~48--57, \url{https://doi.org/10.1109/TPAMI.2013.135}.

\bibitem{tang2018limit}
{\sc M.~Tang, C.~E. Priebe, et~al.}, {\em Limit theorems for eigenvectors of
  the normalized laplacian for random graphs}, The Annals of Statistics, 46
  (2018), pp.~2360--2415.

\bibitem{tang2013universally}
{\sc M.~Tang, D.~L. Sussman, C.~E. Priebe, et~al.}, {\em Universally consistent
  vertex classification for latent positions graphs}, The Annals of Statistics,
  41 (2013), pp.~1406--1430.

\bibitem{vershynin2018high}
{\sc R.~Vershynin}, {\em High-dimensional probability: An introduction with
  applications in data science}, vol.~47, Cambridge university press, 2018.

\bibitem{wang2019joint}
{\sc S.~{Wang}, J.~{Arroyo}, J.~T. {Vogelstein}, and C.~E. {Priebe}}, {\em
  Joint embedding of graphs}, IEEE Transactions on Pattern Analysis and Machine
  Intelligence,  (2019).

\bibitem{waugh2011party}
{\sc A.~S. Waugh, L.~Pei, J.~H. Fowler, P.~J. Mucha, and M.~A. Porter}, {\em
  Party polarization in congress: A network science approach}, 2011,
  \url{https://arxiv.org/abs/0907.3509}.

\bibitem{yang2020simultaneous}
{\sc C.~Yang, C.~E. Priebe, Y.~Park, and D.~J. Marchette}, {\em Simultaneous
  dimensionality and complexity model selection for spectral graph clustering},
  Journal of Computational and Graphical Statistics,  (2020), pp.~1--20.

\bibitem{yoder2020vertex}
{\sc J.~Yoder, L.~Chen, H.~Pao, E.~Bridgeford, K.~Levin, D.~E. Fishkind,
  C.~Priebe, and V.~Lyzinski}, {\em Vertex nomination: The canonical sampling
  and the extended spectral nomination schemes}, Computational Statistics \&
  Data Analysis, 145 (2020), p.~106916.

\bibitem{young2007random}
{\sc S.~J. Young and E.~R. Scheinerman}, {\em Random dot product graph models
  for social networks}, in International Workshop on Algorithms and Models for
  the Web-Graph, Springer, 2007, pp.~138--149.

\bibitem{zhang2007binary}
{\sc Z.~Zhang, T.~Li, C.~Ding, and X.~Zhang}, {\em Binary matrix factorization
  with applications}, in Seventh IEEE international conference on data mining
  (ICDM 2007), IEEE, 2007, pp.~391--400.

\bibitem{zhao2020note}
{\sc Y.~Zhao}, {\em A note on new bernstein-type inequalities for the
  log-likelihood function of bernoulli variables}, Statistics \& Probability
  Letters,  (2020), p.~108779.

\end{thebibliography}
\end{document}